\PassOptionsToPackage{boxed,vlined,linesnumbered}{algorithm2e}
\documentclass[arx]{jmlr}

\usepackage{longtable}

\usepackage{booktabs}

\makeatletter
\let\Ginclude@graphics\@org@Ginclude@graphics 
\makeatother

\jmlrvolume{}
\jmlryear{2022}

\title[Value Function Approximations via Kernel Embeddings]{Value Function Approximations via Kernel Embeddings for No-Regret Reinforcement Learning}

 \author{%
 \Name{Sayak Ray Chowdhury} \Email{sayak@bu.edu}\\
 \addr Boston University
 \AND
 \Name{Rafael Oliveira} \Email{rafael.oliveira@sydney.edu.au}\\
 \addr The University of Sydney%
}

\usepackage[toc,page]{appendix}
\usepackage{amssymb,amsmath,amsfonts}
\usepackage{mathrsfs}
\usepackage{dsfont}
\usepackage{wasysym}
\usepackage{mathnotations}
\usepackage{bigints}
\usepackage{mathtools}

\newcommand{\norm}[1]{\left\lVert#1\right\rVert}

\newcommand{\esp}{\mathbb{E}}
\newcommand{\prob}[1]{\mathbb{P}\left[{#1}\right]}
\newcommand{\given}{\; \big\vert \;}

\newcommand{\inner}[2]{\left \langle #1, #2 \right \rangle }

\graphicspath{{Figures/}}

\newtheorem{mytheorem}{Theorem}

\newtheorem{mylemma}{Lemma}

\newtheorem{myassumption}{Assumption}

\newtheorem{myremark}{Remark}

\newcommand{\beq}{\begin{equation}}
\newcommand{\eeq}{\end{equation}}
\newcommand{\beqn}{\begin{equation*}}
\newcommand{\eeqn}{\end{equation*}}
\newcommand{\beqa}{\begin{eqnarray}}
\newcommand{\eeqa}{\end{eqnarray}}
\newcommand{\beqan}{\begin{eqnarray*}}
\newcommand{\eeqan}{\end{eqnarray*}}

\newcommand{\argmax}{\mathop{\mathrm{argmax}}}

\newcommand*{\operator}[1]{\operatorname{#1}}
\newcommand*{\cme}{\vartheta}

\newcommand*{\iid}{i.i.d.\xspace}

\renewcommand{\vec}[1]{{\boldsymbol{\mathbf{#1}}}}

\newcommand*{\tr}{{\operatorname{tr}}}

\DeclareMathOperator{\HS}{HS}

\newcommand{\CMERL}{\ensuremath{\mathrm{CME}\text{-}\mathrm{RL}}}

\begin{document}

\setlength{\belowdisplayskip}{3pt} \setlength{\belowdisplayshortskip}{3pt}
\setlength{\abovedisplayskip}{3pt} \setlength{\abovedisplayshortskip}{3pt}

\maketitle

\begin{abstract}
We consider the regret minimization problem in reinforcement learning (RL) in the episodic setting.
In many real-world RL environments, the state and action spaces are continuous or very large. Existing approaches establish regret guarantees by either a low-dimensional representation of the stochastic transition model or an approximation of the $Q$-functions. However, the understanding of function approximation schemes for state-value functions largely remains missing. In this paper, we propose an online model-based RL algorithm, namely the CME-RL, that learns representations of transition distributions as embeddings in a reproducing kernel Hilbert space while carefully
balancing the exploitation-exploration tradeoff. We demonstrate the efficiency of our algorithm
by proving a frequentist (worst-case) regret bound that is of order $\tilde{O}\big(H\gamma_N\sqrt{N}\big)$\footnote{ $\tilde{O}(\cdot)$ hides only absolute constant and poly-logarithmic factors.}, where $H$ is the episode length, $N$ is the total number of time steps and $\gamma_N$ is an information theoretic quantity relating the effective dimension of the state-action feature space. Our method bypasses the need for estimating transition probabilities and applies to any domain on which kernels can be defined. It also brings new insights into the general theory of kernel methods for approximate inference and RL regret minimization.
\end{abstract}


\vspace*{-3mm}
\section{Introduction}

Reinforcement learning (RL) is concerned with learning
to take actions to maximize rewards, by trial and error, in
environments that can evolve in response to actions. A Markov decision process (MDP) \citep{puterman2014markov} is a popular framework to model decision making in RL environments. In the MDP, starting from
an initial observed state, an agent repeatedly (a)
takes an action, (b) receives a reward, and (c) observes the next
state of the MDP. The traditional RL objective is a \emph{search} goal – find a \emph{policy} (a rule to select an action for each state)
with high total reward using as few interactions with the
environment as possible, also known as the
sample complexity of RL \citep{strehl2009reinforcement}. This is, however, quite different
from the corresponding \emph{optimization} goal, where the learner
seeks to maximize the total reward earned from all its decisions,
or equivalently, minimize the \emph{regret} or shortfall in total
reward compared to that of an optimal policy \citep{jaksch2010near}. This objective
is relevant in many practical sequential decision-making
settings in which every decision that is taken carries utility
or value – recommendation systems, sequential investment
and portfolio allocation, dynamic resource allocation in communication
systems etc. In such \emph{online} optimization settings, there is
no separate budget or time devoted to purely exploring the
unknown environment; rather, exploration and exploitation
must be carefully balanced.

\subsection{Related work}
Several studies have considered the task of regret minimization in \emph{tabular} MDPs, in which the state and action spaces are finite, and the value function is represented by a table \citep{jaksch2010near,osband2013more,azar2017minimax,dann2017unifying,jin2018q,efroni2019tight,zanette2019tighter}. The regret bound achieved by these works essentially is proportional to $\sqrt{SAN}$, where $S$ and $A$ denote the
numbers of states and actions, respectively, and $N$ the total number of steps. 
In many practical applications, however, the number of states and actions is enormous. For example, the game of Go has a state space with size $3^{361}$, and the state and action spaces of certain robotics applications can even be continuous. These continuous state and action spaces make RL a challenging
task, especially in terms of generalizing learnt knowledge across unseen states and actions. In such cases, the tabular model suffers from the ``curse of dimensionality" problem. To tackle this issue, the popular ``optimism in the face of uncertainty" principle from \citet{jaksch2010near} has been extended to handle continuous MDPs,
when assuming some Lipschitz-like smoothness or regularity on the rewards and dynamics \citep{ortner2012online,domingues2020regret}.

Another line of work considers \emph{function approximation}, i.e., they use features to parameterize reward and transition models, with the hope that the features
can capture leading structures of the MDP \citep{osband2014model,chowdhury2019online}. The \emph{model-based} algorithms developed in these works assume oracle access to an optimistic planner to facilitate the learning. The optimistic planning step is quite prohibitive and often becomes computationally intractable for continuous state and action spaces. \citet{yang2019reinforcement} consider a low-rank bilinear transition model bypassing the complicated planning step; however, their algorithm potentially needs to compute the \emph{value function} across all states. This suffers an $\Omega(S)$ computational complexity and as a consequence cannot directly
handle continuous state spaces.
\cite{ayoub2020model} consider linear-mixture transition structure
that includes the bilinear model as a special case. However, their algorithm too suffers the 
$\Omega(S)$ computational complexity. To alleviate the computational burden intrinsic to these model-based approaches, a recent body of work parameterizes the \emph{value functions} directly, using $d$-dimensional state-action
feature maps, and develop \emph{model-free} algorithms bypassing the need for fully
learning the reward and transition models \citep{jin2019provably,wang2019optimism,zanette2020frequentist}. Under the assumption that the (action-)value function can be approximated by a linear or a generalized linear function of the feature vectors, these papers develop algorithms with regret bound proportional to $\text{poly}(d)\sqrt{T}$, which is independent of the size of the state and action spaces. \cite{wang2020provably} generalizes this approach by designing an algorithm that works with general (non-linear) value function approximators and prove a similar regret guarantee that depends on the eluder dimension \citep{russo2013eluder} and log-covering number of the
underlying function class.

A few recent works have proposed kernel-based value function approximation algorithms. \cite{yang2020provably} consider kernel and neural function approximations and designed algorithms with regret characterized by intrinsic complexity of the function classes. More closely related to our work, \citet{Domingues2021KBRLfinite} recently proposed a kernel-based RL algorithm via value function approximation. Their main assumption relies on Lipschitz continuity of the reward functions and the state transition kernels. In contrast to their work, we are able to obtain tighter regret bounds by applying typical assumptions in the kernel embeddings literature, which we show are satisfied for a variety of practical systems. Nevertheless, there is a lack of theoretical understanding in designing provably efficient model-based RL algorithms with (non-linear) value function approximation, which we aim to address.


\subsection{Contributions}
In this work, we revisit function approximation in RL by modeling
the value functions as elements of a reproducing kernel Hilbert space (RKHS) \citep{Scholkopf2002} compatible with a (possibly infinite dimensional) state feature map. 
The main motivation behind this formulation is that the conditional expectations of any function in the RKHS become a linear operation, via the RKHS inner product with an appropriate distribution embedding, known as the \emph{conditional mean embedding} \citep{Muandet2016}.
In recent years, conditional mean embeddings (CMEs) have found extensive applications in many machine learning tasks \citep{song2009hilbert,song2010hilbert,song2010nonparametric,Song2013,fukumizu2008kernel,fukumizu2009kernel, Hsu2019, chowdhury2020active}. The foremost advantage of CMEs in our setup is that one can directly compute conditional expectations of the value functions based only on the observed data, since the alternative approach of estimating the transition probabilities as an intermediate step scales poorly with the dimension of the state space \citep{grunewalder2012modelling}. The convergence of conditional mean estimates to the true embeddings in the RKHS norm has been established by \citet{grunewalder2012modelling} assuming access to \emph{independent and identically distributed} (\iid) transition samples (the ``simulator" setting). However, in the online RL environment  like the one considered in this work, one collects data based on past observations, and hence the existing result fails to remain useful. Against this backdrop, we make the following contributions:
\begin{itemize}
    \item In the online RL environment, we derive a concentration inequality for mean embedding estimates of the transition distribution around the true embeddings as a function of the uncertainties around these estimates (\autoref{lem:concentration}). This bound not only serves as a key tool in designing our model-based RL algorithm but also is of independent interest.
    
    \item Focusing on the value function approximation in the RKHS setting, we present the first model based RL algorithm, namely the \emph{Conditional Mean Embedding RL} (\CMERL), that is provably efficient in  regret performance and does not require any additional oracle access or stronger computational assumptions (\autoref{alg:main}). Concretely, in the general setting of an episodic MDP, we prove that \CMERL\ enjoys a regret bound of $\tilde{O}(H\gamma_N\sqrt{N})$, where $H$ is the length of each episode, $\gamma_N$ is a complexity measure relating the effective dimension of the RKHS compatible with the state-action features (\autoref{thm:cumulative-regret-final}). 
    
    \item Our approach is also robust to the RKHS modelling assumption: when the value functions are not elements of the RKHS, but $\zeta$-close to some RKHS element in the $\ell_\infty$ norm, then (a modified version of) \CMERL\ achieves a $\tilde{O}(H\gamma_N\sqrt{N} +\zeta N)$ regret, where the linear regret term arises due to the function class misspecification (\autoref{thm:cumulative-regret-final-approx}).
\end{itemize}



\vspace*{-3mm}
\section{Preliminaries}
\paragraph{Notations} We begin by introducing some notations. 
We let $\cH$ be a arbitrary Hilbert space with inner product $\inner{\cdot}{\cdot}_{\cH}$ and corresponding norm $\norm{\cdot}_{\cH}$. When $\cG$ is another Hilbert space, we denote by $\cL(\cG,\cH)$ the Banach space of bounded linear operators from $\cG$ to $\cH$, with the operator norm $\norm{\operator A} := \sup_{\norm{g}_{\cG}=1} \norm{\operator A g}_{\cH}$. We let $\HS(\cG, \cH)$ denote the subspace of operators in $\cL(\cG,\cH)$ with bounded Hilbert-Schmidt norm $\norm{\operator A}_{\HS}:=\big(\sum_{i,j=1}^{\infty}\inner{f_i}{\operator A g_j}_{\cH}^2\big)^{1/2}$, where the $f_i$'s form a complete orthonormal system (CONS) for $\cH$ and the $g_j$'s form a CONS for $\cG$. In the case $\cG=\cH$, we set $\cL(\cH):=\cL(\cH,\cH)$. We denote by $\cL_{+}(\cH)$ the set of all bounded, positive-definite linear operators on $\cH$, i.e., $\operator A \in \cL_{+}(\cH)$ if, for any $h \in \cH$, $\inner{h}{\operator A h}_\cH > 0$.

 
\paragraph{Regret minimization in finite-horizon episodic MDPs}
We consider episodic reinforcement learning in a finite-horizon Markov decision process (MDP) of episode length $H$ with (possibly infinite) state and action spaces $\cS$ and $\cA$, respectively, reward function $R:\cS \times \cA \ra [0,1]$, and transition probability measure $P : \cS \times \cA \ra \Delta(\cS)$, where $\Delta(\cS)$ denotes the probability simplex on $\cS$.
The learning agent interacts with the MDP in episodes and, at each episode $t$, a trajectory $(s_1^t,a_1^t,r_1^t,\ldots,s_H^t,a_H^t,r_H^t,s_{H+1}^t)$ is generated. Here $a_h^t$ denotes the action taken at state $s_h^t$, $r_h^t:=R(s_h^t,a_h^t)$ denotes the immediate reward, and  $s_{h+1}^t\sim P(\cdot | s_h^t,a_h^t)$ denotes the random next state. The initial state $s_1^t$ is assumed to be fixed and history independent, and can even be possibly chosen by an adversary. The episode terminates when $s_{H+1}^t$ is reached, where the
agent cannot take any action and hence receives no reward. The actions are chosen following some policy $\pi=(\pi_1,\ldots,\pi_H)$, where each $\pi_h$ is a mapping from the state space $\cS$ into the action space $\cA$. The agent would like to find a policy $\pi$ that maximizes the long-term expected cumulative reward starting from every state $s \in \cS$ and every step $h \in [H]$, defined as:
\begin{align*}
V_{h}^{\pi}(s):= \mathbb{E}\left[\sum\nolimits_{j=h}^{H}R\left(s_j,\pi_j(s_j)\right)\given s_h=s\right]~. 
\end{align*}
We call $V_{h}^{\pi}:\cS \ra \Real$ the value function of policy $\pi$ at step $h$.
Accordingly, we also define the action-value function, or $Q$-function, $Q_h^{\pi}:\cS \times \cA \ra \Real$ as:
\beqn
Q_h^{\pi}(s,a) := R(s,a) + \mathbb{E}\left[\sum\nolimits_{j=h+1}^{H}R\left(s_j,\pi_j(s_j)\right)\given s_h=s,a_h=a\right]~,
\eeqn
which gives the expected value
of cumulative rewards starting from a state-action pair at the $h$-th step and following the
policy $\pi$ afterwards. Note that $V_h^{\pi}(s) = Q_h^{\pi}(s,\pi_h(s))$ and it satisfies the Bellman equation:
\beq
\label{eq:Bellman}
V_h^{\pi}(s)=R(s,\pi_h(s))+\esp_{X \sim P(\cdot |s,\pi_h(s))}\left[V_{h+1}^{\pi}(X)\right], \quad \forall h \in [H]~,
\eeq
with $V_{H+1}^{\pi}(s) = 0$ for all $s \in \cS$.
We denote by $\pi^\star$ an optimal policy satisfying:
\begin{align*}
V^{\pi^\star}_{h}(s)=\max_{\pi \in \Pi}V^{\pi}_{h}(s), \quad \forall s \in \cS,\; \forall h \in [H],
\end{align*}
where $\Pi$ is the set of all non-stationary policies. Since the episode length is finite, such a policy exists when the action space $\cA$ is large but finite \citep{puterman2014markov}. We denote the optimal value function by $V^{\star}_{h}(s):=V_h^{\pi^\star}(s)$. We also denote the optimal action-value function (or $Q$-function) as $Q_h^\star(s,a) = \max_{\pi} Q_h^\pi(s,a)$. It is easily shown that the optimal action-value function satisfies the Bellman optimality equation:
\beq
\label{eq:opt-Bellman}
Q_h^{\star}(s,a):=R(s,a)+\esp_{X \sim P(\cdot |s,a)}\left[V^\star_{h+1}(X)\right], \quad \forall h \in [H]~,
\eeq
with $V^\star_h(s)=\max_{a \in \cA}Q^\star_h(s,a)$. This implies that the optimal policy is the greedy policy with respect to the optimal action-value functions. Thus, to find the optimal policy $\pi^\star$, it suffices to estimate the optimal action-value functions $(Q_h^\star)_{h \in [H]}$.

The agent aims to learn the optimal policy by interacting
with the environment during a set of episodes. We measure performance of the agent by the cumulative (pseudo) regret accumulated over $T$ episodes, defined as:
\begin{align*}
\cR(N):=\sum\nolimits_{t=1}^{T} \left[V^{\star}_{1}(s_1^t)-V^{\pi^t}_{1}(s_1^t)\right],
\end{align*}
where $\pi^t$ is the policy chosen by the agent at episode $t$ and $N=TH$ is the total number of steps. The regret measures the quantum of
reward that the learner gives up by not knowing the MDP
in advance and applying the optimal policy $\pi^\star$ from the start.
We seek algorithms that attain sublinear regret $\cR(N)=o(N)$ in the number of steps
they face, since, for instance, an algorithm that does not adapt
its policy selection behavior depending on past experience
can easily be seen to achieve linear ($\Omega(N)$) regret \citep{lai1985asymptotically}.


\paragraph{Value function approximation in episodic MDPs}
A very large or possibly infinite state and action space makes reinforcement learning a challenging task. To obtain sub-linear regret guarantees, it is necessary to posit some regularity assumptions on the underlying function class. In this paper, we use reproducing kernel Hilbert spaces to model the value functions. Let $\cH_\psi$ and $\cH_\phi$ be two RKHSs with continuous positive semi-definite kernel functions $k_\psi:\cS \times \cS \to \Real_{+}$ and $k_\phi:(\cS \times \cA) \times (\cS \times \cA) \to \Real_{+}$, with corresponding inner products $\inner{\cdot}{\cdot}_{\cH_\psi}$ and $\inner{\cdot}{\cdot}_{\cH_\phi}$, respectively. There exist feature maps $\psi:\cS \to \cH_\psi$ and $\phi:\cS \times \cA \to \cH_\phi$ such that $k_\psi(\cdot,\cdot)=\inner{\psi(\cdot)}{\psi(\cdot)}_{\cH_\psi}$ and $k_\phi(\cdot,\cdot)=\inner{\phi(\cdot)}{ \phi(\cdot)}_{\cH_\phi}$, respectively \citep{steinwart2008support}.

The weakest assumption one can pose on the value functions is realizability, which posits that the optimal value functions $(V_h^\star)_{h \in [H]}$ lie in the RKHS $\cH_\psi$, or at least are well-approximated by $\cH_\psi$. For stateless MDPs or multi-armed
bandits where $H=1$, realizability alone suffices for provably efficient algorithms \citep{abbasi2011improved,chowdhury2017kernelized}. But it does not seem to be sufficient when $H > 1$, and in these settings it is common to make stronger assumptions \citep{jin2019provably,wang2019optimism,wang2020provably}. Following these works, our main 
assumption is a closure property
for all value functions in the following class:
\begin{align}\label{eq:function-class}
    \cV := \left\lbrace s \mapsto \min \left\lbrace H, \max_{a\in\cA} \left\lbrace R(s,a) + \inner{\phi(s,a)}{\mu}_{\cH_\phi} + \eta \sqrt{\inner{\phi(s,a)}{\Sigma^{-1}\phi(s,a)}_{\cH_\phi}} \right\rbrace \right\rbrace \right\rbrace,
\end{align}
where $0 < \eta < \infty$, $\mu \in \cH_\phi$ and $\Sigma \in \cL_{+}(\cH_\phi)$ are the parameters of the function class. 

\begin{myassumption}[Optimistic closure] \label{ass:opt-closure}
For any $V \in \cV$ (cf. \autoref{eq:function-class}), we have $V \in \cH_\psi$. Furthermore, for a positive constant $B_V$, we have $\norm{V}_{\cH_\psi} \leq B_V$.
\end{myassumption}
While this property seems quite strong, we note
that related closure-type assumptions are common in the literature. We will relax this assumption later in \autoref{sec:misspec}.
In addition, our results do not require explicit knowledge of $\cH_\psi$ nor its kernel $k_\psi$, as we will only interact with elements of $\cV$ via point evaluations and RKHS norm bounds.


\vspace*{-3mm}
\section{RKHS embeddings of transition distribution} 
In order to find an estimate of the optimal value function, it is imperative to estimate the conditional expectations of the form $\mathbb{E}_{X \sim P(\cdot|s,a)}[f(X)]$. In the model-based approach considered in this work, we do so by estimating the mean embedding of the conditional distribution $P(\cdot|s,a)$, which is the focus of this section.
For a bounded kernel\footnote{Boundedness of a kernel holds for any stationary kernel, e.g., the \emph{squared exponential} kernel and the \textit{Mat$\acute{e}$rn} kernel \citep{Rasmussen2006}.} $k_\psi$ on the state space $\cS$, the mean embedding of the conditional distribution $P(\cdot|s,a)$ in $\cH_\psi$ is an element $\cme_P^{(s,a)} \in \cH_\psi$ such that:
\begin{equation}
    \forall f \in \cH_\psi,\quad  \esp_{X \sim P(\cdot | s,a)}[f(X)]= \inner{f}{ \cme_P^{(s,a)}}_{\cH_\psi}~.
    \label{eq:cme}
\end{equation}
The mean embedding can be explicitly expressed as a function:
\beqn
\cme_P^{(s,a)}(y)=\esp_{X \sim P(\cdot | s,a)}\left[k_\psi(X,y)\right],
\eeqn
for all $y \in \cS$. 
If the kernel $k_\psi$ is characteristic, such as a stationary kernel, then the mapping $P(\cdot|s,a)\mapsto\cme_P^{(s,a)}$ is injective, defining a one-to-one relationship between transition distributions and elements of $\cH_\psi$ \citep{Sriperumbudur2011}. 
Following existing works \citep{song2009hilbert,grunewalder2012modelling}, we now make a smoothness assumption
on the transition distribution. 
\begin{myassumption}\label{ass:dist-smoothness}
For any $f\in\cH_\psi$,
the function $(s,a)\mapsto\! \esp_{X \sim P(\cdot | s,a)}\left[f(X)\right]$ lies in $\cH_\phi$.
\end{myassumption}
Under \autoref{ass:dist-smoothness}, the mean embeddings admit a linear representation in state-action features via the \emph{conditional embedding operator}
$\Theta_P \in \cL(\cH_\phi,\cH_\psi)$ such that:
\begin{equation}
     \forall (s,a) \in \cS \times \cA, \quad\cme_P^{(s,a)} = \Theta_P \phi(s,a)~.
    \label{eq:CMO}
\end{equation}
\autoref{ass:dist-smoothness} always holds for finite domains with characteristic kernels. Though it is not necessarily true for continuous domains, we note that the CMEs for classical linear \citep{abbasi2011regret} and non-linear \citep{kakade2020information} dynamical systems satisfy this assumption.

\subsection{Sample estimate of conditional mean embedding} At the beginning of each episode $t$, given the observations $\cD_{t}:=(s_h^\tau,a_h^\tau,s_{h+1}^\tau)_{\tau < t,h\leq H}$ until episode $t-1$, we consider a sample based estimate of the conditional embedding operator.
This is achieved by solving the following ridge-regression problem:
\beq
    \min_{\Theta \in \HS(\cH_\phi, \cH_\psi)}\sum\nolimits_{\tau < t,h\leq H}\norm{\psi(s_{h+1}^\tau)-\Theta\phi(s^\tau_h,a^\tau_h)}_{\cH_\psi}^2+\lambda\norm{\Theta}_{\HS}^2~,
    \label{eq:operator-regression}
\eeq
where $\lambda > 0$ is a regularising constant. The solution of \autoref{eq:operator-regression} is given by:
\begin{equation}
    \hat\Theta_t = \sum\nolimits_{\tau < t,h\leq H}\psi(s^\tau_{h+1})\otimes\phi(s^\tau_h,a^\tau_h)\left(\hat{\operator C}_{\phi,t} +\lambda  \operator I\right)^{-1}~,
   \label{eq:cmo-estimate}
\end{equation}
where $\hat{\operator C}_{\phi,t}:=\sum_{\tau < t,h\leq H}\phi(s^\tau_h,a^\tau_h)\otimes \phi(s^\tau_h,a^\tau_h)$.
To simplify notations, we now let $n=(t-1)H$ denote the total number of steps completed at the beginning of episode $t$.
We denote a vector $  k_{\phi,t}(s,a) \in \Real^n$ and a matrix $\operator K_{\phi,t} \in \Real^{n \times n}$ by:
\beqn
  k_{\phi,t}(s,a):=\left[k_\phi\left((s_h^\tau,a_h^\tau),(s,a)\right)\right]_{\tau < t,h\leq H}, \quad \operator K_{\phi,t}:=[k_{\phi}((s_h^\tau,a_h^\tau),(s_{h'}^{\tau'},a_{h'}^{\tau'}))]_{\tau,\tau'<t,h,h'\leq H}~.
\eeqn
Then, via \autoref{eq:cmo-estimate}, the conditional mean embeddings can be estimated as
\begin{equation}
    \hat\cme_t^{(s,a)}=\hat\Theta_t\phi(s,a)=\sum\nolimits_{\tau < t,h \leq H}\left[  \alpha_t(s,a)\right]_{(\tau,h)}\psi(s^\tau_{h+1})~,
    \label{eq:cme-estimate}
\end{equation}
where we define the weight vector $  \alpha_t(s,a):=(\operator K_{\phi,t}+\lambda  \operator I)^{-1}  k_{\phi,t}(s,a)$.

\vspace*{-2mm}
\subsection{Concentration of mean embedding estimates}


In this section, we show that for any state-action pair $(s,a)$, the CME estimates $\hat\cme_t^{(s,a)}$ lies within a high-probability confidence region around the true embedding $\cme_P^{(s,a)}$. This eventually translates, via \autoref{eq:cme}, to a concentration property of $\inner{f}{\hat\cme_t^{(s,a)}}_{\cH_\psi}$ around $\esp_{X \sim P(\cdot | s,a)}[f(X)]$ for any $f \in \cH_\psi$. The uncertainty of CME estimates can be characterized by the variance estimate
$\sigma^2_{\phi,t}(s,a) := \lambda \inner{\phi(s,a)}{\operator M_t^{-1}\phi(s,a)}_{\cH_\phi}$, where $\operator M_t:= \hat{\operator C}_{\phi,t}+\lambda  \operator I$. To see this, note that an application of Sherman-Morrison formula yields:
\beq
\sigma_{\phi,t}^2(s,a) := k_{\phi}((s,a),(s,a))-  k_{\phi,t}(s,a)^\top (\operator K_{\phi,t}+\lambda  \operator I)^{-1}  k_{\phi,t}(s,a)~,
\label{eq:predictive-variance}
\eeq
which is equivalent to the predictive variance of a Gaussian process (GP) \citep{Rasmussen2006}. Although a sample from a GP is usually not an element of the RKHS defined by its kernel \citep{Lukic2001}, the following result allows us to use $\sigma_{\phi,t}^2(s,a)$ as an error measure. 

\vspace*{-2mm}
\begin{mytheorem}[Concentration of the conditional embedding operator]\label{lem:concentration}
Suppose that $\sup_{s\in \cS}\sqrt{k_\psi(s,s)} \leq B_\psi$. Then, under \autoref{ass:dist-smoothness}, for any $\lambda >0$ and $\delta \in (0,1]$,
\beqn
\prob{\forall t \in \Nat,\;\; \norm{\big(\Theta_P-\hat\Theta_t\big)\operator M_t^{1/2}}  \leq \beta_t(\delta) } \geq 1-\delta~,
\eeqn
where $\beta_t(\delta):= \sqrt{2\lambda B_P^2+ 256(1+\lambda^{-1})\log (\det( \operator I+\lambda^{-1}\operator K_{\phi,t})^{1/2})\log(2t^2H/\delta)}$, $B_P \geq \norm{\Theta_P}_{\HS}$.
\end{mytheorem}
\autoref{lem:concentration} implies a concentration inequality for the CME estimates, since, for all $t\geq 1$: 
\beqn
\norm{\cme_{P}^{(s,a)}-\hat\cme_t^{(s,a)}}_{\cH_\psi} \leq \norm{(\Theta_{P}-\hat\Theta_t)\operator M_t^{1/2}}\norm{\phi(s,a)}_{\operator M_t^{-1}} \leq \beta_t(\delta)\lambda^{-1/2} \sigma_{\phi,t}(s,a),\,\forall (s,a) \in \cS\times\cA,
\eeqn
with probability at least $1-\delta$. This forms the core of our value function approximations.

\vspace*{-2mm}
\begin{myremark}
Considering the simulation setting, \cite{grunewalder2012modelling} assume access to a sample $(s_i,a_i,s'_i)_{i=1}^{m}$, drawn i.i.d. from a joint distribution $P_{0}$ such that the conditional probabilities satisfy $P_{0}(s'_i|s_i,a_i)=P(s'_i|s_i,a_i), \forall i$. 
Under \autoref{ass:dist-smoothness}, they establish the convergence of CME estimates $\hat\cme_t^{(s,a)}$ to the true CMEs $\cme_P^{(s,a)}$ in $P_0$-probability. This guarantee, however, does not apply to our setting, since we do not assume any simulator access.
\end{myremark}

\vspace*{-2mm}
\paragraph{Proof sketch of \autoref{lem:concentration}} 
To derive this result, we note that the sequence evaluation noise $\epsilon_{h}^t := \psi(s_{h+1}^t) - \Theta_P\phi(s_h^t,a_h^t)$ at each step $h$ of episode $t$ forms a martingale difference sequence, with each element having a bounded RKHS norm. We overload notation to define, for each pair $(t,h)$, the operator $\operator M_{t,h}=\operator M_t+\sum_{j\leq h}\phi(s^t_j,a^t_j)\otimes \phi(s^\tau_j,a^\tau_j)$, and the estimate
\begin{equation*}
    \hat\Theta_{t,h} = \left(\sum\nolimits_{\tau < t,h\leq H}\psi(s^\tau_{h+1})\otimes\phi(s^\tau_h,a^\tau_h)+\sum\nolimits_{j \leq h}\psi(s^t_{j+1})\otimes\phi(s^t_j,a^t_j)\right)\operator M_{t,h}^{-1}~.
\end{equation*}
Now, we consider the random variable
 $z_{t,h} = \norm{(\hat\Theta_{t,h}-\Theta_P)\operator M_{t,h}^{1/2}}^2_{\HS}$, and prove a high-probability upper bound on it using Azuma-Hoeffding's inequality for martingales. In fact, we show that $z_{t,h} \leq \beta^2_{t,h}(\delta)$ uniformly over all pair $(t,h)$ with probability at least $1-\delta$, where $\beta_{t,h}(\delta)$ is defined similarly to $\beta_t(\delta)$ with only $\operator K_{\phi,t}$ being replaced by $\operator K_{\phi,t,h}:=[k_{\phi}((s_j^\tau,a_j^\tau),(s_{j'}^{\tau'},a_{j'}^{\tau'}))]_{(\tau,j),(\tau',j')\leq (t,h)}$ -- the gram-matrix at step $h$ of episode $t$. The proof then follows by noting that
    $\norm{(\hat\Theta_{t}-\Theta_P)\operator M_{t}^{1/2}}_{\HS}=z_{t-1,H}^{1/2}\leq \beta_{t-1,H}(\delta) \leq \beta_t(\delta)$. The complete proof is given in the supplementary material.




\vspace*{-3mm}
\section{RL exploration using RKHS embeddings}

In this section, we aim to develop an online RL algorithm using the conditional mean embedding estimates that balances exploration and exploitation (near) optimally. We realize this, at a high level, by following the Upper-Confidence Bound (UCB) principle and thus our algorithm falls in a similar framework as in \citet{jaksch2010near,azar2017minimax,yang2019reinforcement}.

\vspace*{-2mm}
\subsection{The Conditional Mean Embedding RL (\CMERL) algorithm}

At a high level, each episode $t$
consists of two passes over all steps. In the first pass, we maintain the $Q$-function estimates via dynamic programming. To balance the exploration-exploitation trade-off, we first define a confidence set $\cC_t$ that contains the set of conditional embedding operators that are deemed to be consistent with all the data that has been collected in the past. Specifically, for any $\delta \in (0,1]$, $\lambda > 0$ and  constants $B_P$ and $B_\psi$, \autoref{lem:concentration} governs us to define the confidence set
\beq
\cC_t := \left\lbrace \Theta \in \cL(\cH_\phi,\cH_\psi):\norm{(\Theta-\hat\Theta_t)\operator M_t^{1/2}} \leq \beta_t(\delta/2) \right\rbrace~,
\label{eq:confidence-set}
\eeq
where $\beta_t(\cdot)$ governs the exploration-exploitation trade-off. This confidence set is then used to compute the optimistic $Q$-estimates, starting with $V_{H+1}^{t}(s)=0$, and setting:
\beqa
    \text{for}\; h =H,H\!-\!1,\ldots,1, \quad V_h^t(s) &=& \min\left\lbrace H,\max_{a \in \cA} Q_h^t(s,a)\right\rbrace, \label{eq:optimistic-values}\\
    Q_h^t(s,a) &=& R(s,a)+ \max_{\Theta_{P'} \in \cC_t} \esp_{X \sim P'(\cdot | s,a)}\left[V_{h+1}^{t}(X)\right]~. 
    \label{eq:optimistic-planning}
\eeqa
We note here that we only
require an optimistic estimate of the optimal $Q$-function. Hence, it is not necessary to solve the maximization problem in \autoref{eq:optimistic-planning} explicitly. In fact, we can use a closed-form expression instead of searching for the optimal embedding operator $\Theta_{P'}$ in the confidence set $\cC_t$. If the value estimate $V_{h+1}^t$ lies in the RKHS $\cH_\psi$, we then have from \autoref{eq:cme} that $\esp_{X \sim P'(\cdot | s,a)}\left[V_{h+1}^{t}(X)\right]=\inner{V_{h+1}^t}{\cme_{P'}^{s,a}}_{\cH_\psi}$, and from \autoref{eq:cme-estimate} that:
\beqn
\inner{V_{h+1}^t}{\hat\cme_t^{(s,a)}}_{\cH_\psi}=  \alpha_t(s,a)^\top  v_{h+1}^t=  k_{\phi,t}(s,a)^\top (\operator K_{\phi,t}+\lambda  \operator I)^{-1}  v_{h+1}^t~,
\eeqn
where we define the vector $  v^t_{h+1}:=[V_{h+1}^{t}(s_{h'+1}^\tau)]_{\tau < t,h' \leq H}$. Now, since the confidence set $\cC_t$ is convex, the $Q$-updates given by \autoref{eq:optimistic-planning} admit the closed-form expression:
\beq
Q_h^t(s,a) = R(s,a) +   k_{\phi,t}(s,a)^\top (\operator K_{\phi,t}+\lambda  \operator I)^{-1}  v^t_{h+1} + \norm{V_{h+1}^t}_{\cH_\psi}\beta_t(\delta/2)\lambda^{-1/2} \sigma_{\phi,t}(s,a)~.
\label{eq:efficient-optimistic-planning}
\eeq
We now note that, by the optimistic closure property (\autoref{ass:opt-closure}), the value estimate $V_h^t$ given by \autoref{eq:optimistic-values} lies in the RKHS $\cH_\psi$, rendering the closed-form expression in \autoref{eq:efficient-optimistic-planning} valid. 

In the second pass, we execute the greedy policy with respect to the $Q$-function estimates obtained in the first pass. Specifically,
at each step $h$, we chose the action:
\beq
 a_h^t=\pi^t_h(s^t_h) \in \argmax\nolimits_{a \in \cA} Q_h^t(s_h^t,a)~.
\label{eq:action-selection}
\eeq
The pseudo-code of \CMERL\ is given in \autoref{alg:main}. Note that, in order to implement \CMERL, we do not need to know the kernel $k_\psi$; only the knowledge of the upper bound $B_V$ over the RKHS norm of $V_{h+1}^t$ suffices our purpose. For
simplicity of representation, we assume that the agent, while not knowing the conditional mean embedding operator $\Theta_P$, knows the
reward function $R$. When $R$ is unknown but
an element of the RKHS $\cH_\phi$,
our algorithm can be extended naturally with an optimistic
reward estimation step at each episode, similar
to the contextual bandit setting \citep{chowdhury2017kernelized}. 
\vspace{-5pt}
\begin{algorithm2e}
\caption{Conditional Mean Embedding RL (\CMERL)}
\label{alg:main}
\DontPrintSemicolon
\textbf{Input:} Kernel $k_\phi$, constants $B_P$, $B_V$ and $B_\psi$, parameters $\eta > 0$ and $\delta \in (0,1]$\;
\For{episode $t = 1,\dots,T$}
{
    Receive the initial state $s_1^t$ and set $V_{H+1}^t(\cdot)=0$\;
    \For(\tcp*[h]{Update value function estimates}){step $h=H,\ldots,1$} 
    {
        $Q_h^t(\cdot,\cdot) = R(\cdot,\cdot) +   k_{\phi,t}(\cdot,\cdot)^\top (\operator K_{\phi,t}+\lambda  \operator I)^{-1}  v^t_{h+1} + B_V\beta_t(\delta/2)\lambda^{-1/2} \sigma_{\phi,t}(\cdot,\cdot)$\;
        $V_h^t(\cdot) = \min\left\lbrace H,\max_{a \in \cA} Q_h^t(\cdot,a)\right\rbrace$
    }
    \For(\tcp*[h]{Run episode}){step $h=1,\ldots,H$}
    {
       Take action $a_h^t \in \argmax_{a\in\cA} Q_h^t(s_h^t, a)$ and observe next state $s_{h+1}^t \sim P(\cdot|s_h^t,a_h^t)$
    }
}
\end{algorithm2e}



\paragraph{Computational complexity of \CMERL}
The dominant cost is evaluating the $Q$-function estimates $Q_h^t$ (\autoref{eq:efficient-optimistic-planning}). As typical in kernel methods \citep{Scholkopf2002}, it involves inversion of $tH\times tH$ matrices, which take $O(t^3H^3)$ time. In the policy execution phase (\autoref{eq:action-selection}), 
we do not need to compute the entire $Q$-function as the algorithm
only queries $Q$-values at visited states. Hence, assuming a constant cost of optimizing over the actions,
the per-episode running time is $O(t^3H^4)$. However, using standard sketching techniques like the
Nystr\"{o}m approximation \citep{drineas2005nystrom} or the random Fourier features approximation \citep{RahimiR07random}, and by using the Sherman-Morrison formula to amortize matrix
inversions, per-epsiode running cost can be reduced to $O(m^2H)$, where $m$ is the dimension of feature approximations.

\subsection{Regret bound for \CMERL}
In this section, we present the regret guarantee of our algorithm. We first define 
\beqn
\gamma_N\equiv\gamma_{\phi,\lambda,N}:=\sup_{\cX \subset \cS \times \cA:|\cX|=N}\frac{1}{2}\log \det ( \operator I+\lambda^{-1}\operator K_{\phi,\cX})~,
\eeqn
where $\cX=\lbrace(s_i,a_i)\rbrace_{i \in [N]}$ and $\operator K_{\phi,\cX}=[k_\phi((s_i,a_i),(s_j,a_j))]_{i,j\in [N]}$ is the gram matrix over the data set $\cX$. $\gamma_N$ denotes the \emph{maximum information gain} about a (random) function $f$ sampled from a zero-mean GP with covariance function $k_\phi$ after $N$ noisy observations,
obtained by passing $f$ through an i.i.d. Gaussian
channel $\cN(0,\lambda)$. Consider the case when $k_\phi$ is a squared exponential kernel on $\Real^d$. Then it can be verified that $\gamma_N=O\left((\log N)^{d+1}\right)$ \citep{srinivas2009gaussian}.





\begin{mytheorem}[Cumulative regret of \CMERL]
\label{thm:cumulative-regret-final}
Under assumptions \ref{ass:opt-closure} and \ref{ass:dist-smoothness}, after interacting with the environment for $N = TH$ steps, with probability at least $1-\delta$, \CMERL\ (\autoref{alg:main}) achieves the regret bound
\beqn
\cR(N) \leq  2B_V\alpha_{N,\delta} \sqrt{2(1+\lambda^{-1}B_\phi^2 H)N\gamma_N} + 2H\sqrt{2N\log(2/\delta)}~,
\eeqn
where $B_\phi \geq \sup_{s,a}\sqrt{k_\phi((s,a),(s,a))}$, and 
$\alpha_{N,\delta}:=\sqrt{2\lambda B_P^2+ 256(1+\lambda^{-1})\gamma_N\log(4N^2/\delta)}$.

\end{mytheorem}
\autoref{thm:cumulative-regret-final} yields a $\tilde{O}(H\gamma_N  \sqrt{N})$ regret bound for \CMERL. Comparing to the minimax regret in tabular setting, $\Theta(H\sqrt{SAN})$ \citep{azar2017minimax}, our bound replaces the sublinear dependency on the number of state-action pairs by a linear dependency on the intrinsic complexity measure, $\gamma_N$, of the feature
space $\cH_\phi$, which is crucial in the large state-action
space setting that entails function approximation. Additionally, in the kernelized bandit setting ($H=1$), our bound matches the best known upper bound $O(\gamma_N\sqrt{N})$ \citep{chowdhury2017kernelized}. We note, however, that
while an MDP has state transitions, the bandits do not,
and a naive adaptation of existing kernelized bandit algorithms
to this setting would give a regret exponential in episode length $H$. 
Furthermore, due to the Markov transition structure, the lower bound for kernelized bandits \citep{scarlett2017lower} does not directly apply here. Hence, it remains an interesting
future direction to determine the optimal dependency on $\gamma_N$.

\paragraph{Conversion to PAC guarantee} Similarly to the discussion in \citet{jin2019provably}, our regret bound directly translates
to a sample complexity or probably approximately correct (PAC) guarantee in the following sense. Assuming a fixed initial state $s_1^t=s$ for each episode $t$, with at least a constant probability, we can learn an $\epsilon$-optimal policy $\pi$ that
satisfies $V^\star_1(s)-V^\pi_1(s) \leq \epsilon$ by running \CMERL\ for $T = O(d_{\text{eff}}^2H^2/\epsilon^2)$ episodes, where $d_{\text{eff}}$ is a known upper bound over $\gamma_N$, and then output the greedy policy according to the $Q$-function at $t$-th episode, where $t$ is sampled uniformly from $[T]$. Here $d_{\text{eff}}$ effectively captures the number of significant dimensions of $\cH_\phi$.


\begin{myremark}
\citet{yang2019reinforcement}
assumes the model $P(s'|s,a) = \inner{\psi(s')} {\Theta_P\phi(s,a)}_{\cH_\psi}$, and propose an
algorithm with regret $\tilde{O}(H^2\gamma_N\sqrt{N})$. In comparison, we get an $O(H)$ factor improvement thanks to a tighter control over the sum of predictive variances.
Furthermore, their algorithm can't be implemented exactly as 
they need to apply random sampling to approximate the estimate $\hat \Theta_t$. We overcome this implementational bottleneck by virtue of our novel confidence set construction using the CME estimates (\autoref{lem:concentration}). Moreover, in contrast to \citet{yang2019reinforcement}, our regret guarantee is anytime, i.e., we don't need to know the value of $N$ before
the algorithm runs. 
\end{myremark}


\begin{myremark}  
Considering linear function approximation ($\cH_\phi=\Real^d$), \citet{jin2019provably} assumes that for any $V \in \cV$ (\autoref{eq:function-class}), the map $(s,a)\mapsto \mathbb{E}_{X \sim P(\cdot|s,a)}[V(X)]$ lies in $\cH_\phi$, and propose a model-free
algorithm with regret $\tilde{O}(\sqrt{H^3d^3N})$. For linear kernels, it can be verified that $\gamma_N=O(d\log N)$ and thus our regret (\autoref{thm:cumulative-regret-final}) is of the order $\tilde{O}(Hd\sqrt{N})$. We note that this apparent improvement in our bound is a consequence of slightly stronger assumptions \ref{ass:opt-closure} and \ref{ass:dist-smoothness}. While they obtain the bound by proving a uniform concentration result over the set $\cV$, our result uses a novel concentration property of CME estimates (\autoref{lem:concentration}).
\end{myremark}

\paragraph{Proof sketch of \autoref{thm:cumulative-regret-final}}
We start with a control on the $Q$-function estimates $Q_h^t$, which in turn leads to the regret bound, as our policy is based on $Q_h^t$. We prove that as long as $\Theta_P$ lies in the confidence set $\cC_t$, the $Q$-updates are optimistic estimates of the optimal $Q$-values, i.e., $Q^\star_h(s, a) \leq Q^t_h(s,a)$ for all $(s,a)$, and thus, allow us to pick an optimistic action while sufficiently exploring the state space. 
This implies $V^{\star}_1(s_1^t) \leq V_1^t(s_1^t)$ and thus, in turn, the regret
$\cR(N)  \leq \sum_{t=1}^{T}\big(V_1^t(s_1^t)-V_1^{\pi^t}(s_1^t)\big)$. At this point, we let $g_1^t(s_1^t):=V_1^t(s_1^t)-V_1^{\pi^t}(s_1^t)$ denote the gap between the most optimistic value and the actual value obtained at episode $t$. We then have 
\begin{align*}
 g_1^t(s_1^t) \leq \sum\nolimits_{h=1}^{H}\big(Q_h^t(s_h^t,a_h^t)-\big(R(s_h^t,a_h^t)+\esp_{X \sim P(\cdot|s_h^t,a_h^t)}\left[V_{h+1}^{t}(X)\right]\big) + m_h^t\big),   
\end{align*}
where $(m_h^t)_{t,h}$ denotes a martingale difference sequence. We control this via the Azuma-Hoeffding inequality as $\sum_{t,h}m_{t,h}= O(H\sqrt{N})$. The rest of the terms inside the summation can be controlled, by \autoref{lem:concentration} and by design of the confidence set $\cC_t$, using the predictive variances $\sigma^2_{\phi,t}(s_h^t,a_h^t)$. In fact, for $\Theta_P \in \cC_t$, it holds that \beqan
Q_h^t(s,a)-\left(R(s,a)+\esp_{X \sim P(\cdot|s,a)}\left[V_{h+1}^{t}(X)\right]\right)
 \leq 2B_V\beta_t(\delta/2)\lambda^{-1/2} \sigma_{\phi,t}(s,a) ~.
\eeqan 
Now, the proof can be completed 
by showing that $\sum_{t,h}\sigma_{\phi,t}(s_h^t,a_h^t) = O\big(\sqrt{HN\gamma_N }\big)$. Complete proof of this result is given in the supplementary material.

\vspace*{-1mm}
\subsection{Robustness to model misspecification} 
\label{sec:misspec}

\autoref{thm:cumulative-regret-final} hinges on the fact that any optimistic estimate of the value function can be specified as an element in $\cH_\psi$. In this section, we study the case when there is a misspecification error. Formally, we consider the
following assumption.

\vspace*{-1mm}
\begin{myassumption}[Approximate optimistic closure]\label{ass:approx-opt-closure} There exists constants $\zeta >0$ and $B_V > 0$, such that for any $V \in \cV$ (\autoref{eq:function-class}), there exists a function $\tilde V \in \cH_\psi$ which satisfies $\|V-\tilde V\|_{\infty} \leq \zeta$ and $\|\tilde V\|_{\cH_\psi} \leq B_V$. We call $\zeta$ the misspecification error.
\end{myassumption}
The quality of this approximation will further depend upon how well any $V \in \cV$ can be approximated by a low-norm function in $\cH_\psi$. One specialization is to the
case when $\cV \in \cC_b(\cS)$, the vector space of continuous and bounded functions on $\cS$, and $k_\psi$ is a $\cC_b(\cS)$-universal kernel \citep{steinwart2008support}. In
this case, we can choose $\tilde V$ such that $\|V-\tilde V\|_\infty$ is
arbitrarily small. For technical reasons, we also make the following assumption.
\begin{myassumption}\label{ass:constant-function-in-RKHS}
The RKHS $\cH_\psi$ contains the constant functions.\footnote{This is a mild assumption, because for any RKHS $\cH_\psi$, the direct sum $\cH_\psi + \Real$,
where $\Real$ denotes the RKHS associated with the positive definite kernel $k(s,s')=1$, is again a RKHS with reproducing kernel $k_{\text{new}}(s,s'):=k_\psi(s,s') + 1$.}
\end{myassumption}
The following theorem states that our algorithm is in fact robust to a small
model misspecification. To achieve this, we only need to adopt a different exploration term in \autoref{eq:efficient-optimistic-planning} to account for the misspecification error $\zeta$. To this end, define the $Q$-function updates as
\begin{equation}\label{eq:efficient-optimistic-planning-approx}
    Q_h^t(s,a) := R(s,a) +   k_{\phi,t}(s,a)^\top (\operator K_{\phi,t}+\lambda  \operator I)^{-1}  v^t_{h+1} + \left(B_V + \zeta\norm{  1}_{\cH_\psi}\right)\beta_t(\delta/2)\sigma_{\phi,t}(s,a)~,
\end{equation}
where $\norm{  1}_{\cH_\psi}$ denotes the norm of the all-one function $s\mapsto 1$ in $\cH_\psi$.



\vspace*{-1mm}
\begin{mytheorem}[Cumulative regret under misspecification] \label{thm:cumulative-regret-final-approx} 
Under assumptions  \ref{ass:dist-smoothness}, \ref{ass:approx-opt-closure} and \ref{ass:constant-function-in-RKHS}, with probability at least $1-\delta$, \CMERL\ achieves the regret bound
\beqn
\cR(N) \leq  2\left(B_V+\zeta\norm{  1}_{\cH_\psi}\right)\alpha_{N,\delta} \sqrt{2(1+\lambda^{-1}B_\phi^2 H)N\gamma_N} + 4\zeta N  + 2H\sqrt{2N\log(2/\delta)}~,
\eeqn
where $B_\phi$ and $\alpha_{N,\delta}$ are as given in \autoref{thm:cumulative-regret-final}.
\end{mytheorem}
In comparison with \autoref{thm:cumulative-regret-final}, \autoref{thm:cumulative-regret-final-approx} asserts that \CMERL\ will incur at most an
additional $O(\zeta\gamma_N\sqrt{HN}+\zeta N)$ regret when the model is misspecified. This additional term is linear in $N$ due to the intrinsic bias introduced by the approximation. This linear dependency is standard in the literature, e.g., it is present even in the easier setting of linear function approximation \citep{jin2019provably}. When $\zeta$ is sufficiently small (as is typical for universal kernels $k_\psi$), our algorithm will still enjoy good theoretical guarantees. 

\paragraph{Conversion to PAC guarantee}
Similar to \autoref{thm:cumulative-regret-final}, we can also convert \autoref{thm:cumulative-regret-final-approx} to a PAC guarantee. 
 Assuming a fixed initial state $s$, with at least a constant probability, we can learn an $\epsilon$-optimal policy $\pi$ that satisfies $V^\star_1(s)-V^\pi_1(s) \leq \epsilon + \zeta\gamma_N H^{3/2}$ by running \CMERL\ for $T = O(d_{\text{eff}}^2H^2/\epsilon^2)$ episodes.
  \begin{myremark}[Regret under unknown misspecification error]
 When the misspecification error $\zeta$ is unknown to the agent apriori, one can invoke the dynamic regret balancing scheme of \citet{cutkosky2021dynamic} to get essentially a similar bound as \autoref{thm:cumulative-regret-final-approx} (albeit with a polylog factor blow-up). In fact, \citet{cutkosky2021dynamic} gives a bound for the linear MDP model of \citet{jin2019provably}. Similar techniques can be incorporated to derive a regret bound with unknown $\zeta$ in our setting also.
 \end{myremark}
 \paragraph{Proof sketch of \autoref{thm:cumulative-regret-final-approx}}
 Similar to the proof of \autoref{thm:cumulative-regret-final}, we control the $Q$-function estimates $Q_h^t(s,a)$ (cf. \autoref{eq:efficient-optimistic-planning-approx}), but with necessary modifications taking the effect of the misspecification error $\zeta$ into account. Specifically, we show, for $\Theta_P \in \cC_t$, that 
 \beqan
Q^t_h(s,a) - \!\left(R(s,a)+\esp_{X \sim P(\cdot | s,a)}\left[V_{h+1}^{t}(X)\right]\right)\! 
\leq  2\!\left(\!B_V + \zeta\norm{  1}_{\cH_\psi}\!\!  \right) \lambda^{-1/2}\beta_t(\delta/2)\sigma_{\phi,t}(s,a) + 2\zeta~.
\eeqan
With the result above, we can derive an upper bound on the optimal value $Q^\star_h$ as $Q^\star_h(s, a) \leq Q^t_h(s,a) + 2(H-h)\zeta$, which allows us to pick an optimistic action. The proof then follows similar steps of \autoref{thm:cumulative-regret-final} via control of predictive variances and 
Azuma's inequality.
Complete proof is given in the supplementary material.

\section{Conclusion}
In this paper, we have presented a novel model-based RL algorithm with sub-linear regret guarantees under an optimistic RKHS-closure assumption on the value functions, without requiring a “simulator” access. The algorithm essentially performs an optimistic value iteration step, which is derived from a novel concentration inequality for the mean embeddings of the transition distribution. We have also shown robustness of our algorithm to small model misspecifications. 

As future work, it remains an open research direction to relax the strong optimistic closure assumption to a milder one, as in \citet{zanette2020learning} and \citet{Domingues2021KBRLfinite}, without sacrificing on the computational and regret performances. In terms of computational complexity, \citet{Vial2022misspecified} proposed an algorithm for misspecified linear MDPs with bounded per-iteration computational complexity. Although our method has computational complexity growing with the number of data points, we highlight that constant cost per iteration is possible to be achieved with kernel-based approximations by means of low-rank decompositions \citep{Gijsberts2013}, which is another possible research direction for future work.











\bibliography{2018library,Bandit_RL_bib,main}

\pagebreak

\appendix

\section{Closed-form solution for the operator optimization problem}
In this section, we show that \autoref{eq:efficient-optimistic-planning} is a closed-form solution for the operator optimization problem in \autoref{eq:optimistic-planning}, which is of the form $\max_{\Theta_P \in \cC_t} \Esp_{X\sim P(\cdot|s,a)}[f(X)]$.
Although the term corresponding to $f$ in \autoref{eq:optimistic-planning} is not necessarily in the RKHS $\Hilbert_\psi$, we may for now assume that $f \in \Hilbert_\psi$. In this case, the problem above may be rewritten as a quadratically constrained linear program in $\cL(\cH_\phi,\cH_\psi)$:
\begin{equation}\label{eq:optimization-opnorm}
\max_{\Theta \in \cL(\cH_\phi,\cH_\psi)} \quad \inner{f}{\Theta \phi(s,a)}_{\cH_\psi}
    \quad\mathrm{s.t.} \quad  \norm{(\Theta-\hat\Theta_t)\operator{M}_t^{1/2}} \leq \beta_t(\delta/2) ~,
\end{equation}
where $\operator{M}_t  \in \cL(\cH_\phi, \cH_\phi)$. The problem above admits a closed-form solution by applying the Karush-Kuhn-Tucker (KKT) conditions. Directly oing so would require us to take derivatives of the operator norm. However, observe that for any operator $\operator{M}\in \cL(\cH_\phi,\cH_\psi)$ we have $\norm{\operator M} \leq \sqrt{\tr(\operator M^\top \operator M)}$,
where the latter corresponds to the Hilbert-Schmidt norm. Compared to the operator norm, we can easily take derivatives of the trace to compute the KKT conditions. In addition, due to the upper bound, any solution satisfying the Hilbert-Schmidt norm constraint is also a solution under the operator norm constraint. 
We replace \autoref{eq:optimization-opnorm} with the following problem:
\begin{equation*}
\max_{\Theta \in \cL(\cH_\phi,\cH_\psi)} \quad \inner{f}{\Theta \phi(s,a)}_{\cH_\psi}\quad
    \mathrm{s.t.} \quad  \tr((\Theta-\hat\Theta_t)\operator{M}_t(\Theta-\hat\Theta_t)^\top) \leq \beta_t(\delta/2)^2 ~.
\end{equation*}
Applying the KKT conditions, we solve $\nabla_\Theta \ell(\Theta, \eta) = 0$
with respect to $\Theta \in \cL(\cH_\phi,\cH_\psi)$ and $\eta \in \Real$, $\eta \geq 0$, where $
\ell (\Theta, \eta) := \inner{f}{\Theta \phi(s,a)}_{\cH_\psi} - \eta (\tr((\Theta-\hat\Theta_t)\operator{M}_t(\Theta-\hat\Theta_t)^\top) - \beta_t(\delta/2)^2)$.
First, by $\nabla_\Theta \ell(\Theta, \eta) =0$,  we have
\begin{align}
     f \otimes \phi(s,a) - 2 \eta (\Theta-\hat\Theta_t)\operator{M}_t = 0 
    \implies \Theta = \hat\Theta_t + \frac{1}{2\eta} (f\otimes\phi(s,a))\operator{M}_t^{-1}~.
    \label{eq:kkt-stationarity}
\end{align}
Now, note that, for quadratically constrained linear program, the maximum should lie at the border of the constrained set. Replacing the result above into the constraint, we obtain
\begin{equation*}
    \beta_t(\delta/2)^2 = \frac{1}{4\eta^2}\tr((f\otimes\phi(s,a))\operator{M}_t^{-1}(\phi(s,a)\otimes f)) 
    = \frac{1}{4\eta^2} \inner{\phi(s,a)}{\operator{M}_t^{-1}\phi(s,a)}_{\cH_\phi} \inner{f}{f}_{\cH_\psi}~,
\end{equation*}
so that $\eta = \frac{1}{2\beta_t(\delta/2)}\norm{\phi(s,a)}_{\operator{M}_t^{-1}} \norm{f}_{\cH_\psi}$.
Combining the latter with \autoref{eq:kkt-stationarity}, the solution to \autoref{eq:optimization-opnorm} is then given by
\begin{equation*}
    \Theta_* := \hat\Theta_t + \frac{\beta_t(\delta/2)}{\norm{\phi(s,a)}_{\operator{M}_t^{-1}} \norm{f}_{\cH_\psi}} (f\otimes\phi(s,a))\operator{M}_t^{-1}~,
\end{equation*}
which finally yields
\begin{equation*}
\begin{split}
    \max_{\Theta_P \in \cC_t} \Esp_{X\sim P(\cdot|s,a)}[f(X)] = \inner{f}{\Theta_* \phi(s,a)}_{\cH_\psi}
    &= \inner{f}{\hat\Theta_t \phi(s,a)}_{\cH_\psi}\!\! + \beta_t(\delta/2)\norm{f}_{\cH_\psi} \norm{\phi(s,a)}_{\operator{M}_t^{-1}}\\
    &= \inner{f}{\hat\Theta_t \phi(s,a)}_{\cH_\psi} \!\!+ \beta_t(\delta/2)\lambda^{-1/2}\norm{f}_{\cH_\psi}\sigma_{\phi,t}(s,a)~.
\end{split}
\end{equation*}
Replacing $f$ by $V_{h+1}^t$ in the solution above and adding the reward function (cf. \autoref{eq:optimistic-planning}), we recover \autoref{eq:efficient-optimistic-planning}.

\newpage

\section{ Proof of main results}

\subsection{Proof of \autoref{lem:concentration}}

We first need to define the data-generating process and its properties. Let $\cF_{t,h-1}$ be the filtration induced by the sequence $\cD_{t} \cup \lbrace (s_{j}^t,a_{j}^t)\rbrace_{j\leq h}$, where $\cD_t$ denotes the replay buffer at the beginning of episode $t$. Note that $\cF_{t,0}=\cD_t$. The evaluation noise defined by $\epsilon_{h}^t := \psi(s_{h+1}^t) - \Theta_P\phi(s_h^t,a_h^t)$, for $s_{h+1}^t \sim P(\cdot|s_h^t, a_h^t)$, is such that:
\begin{equation*}
    \esp[\epsilon_h^t | \cF_{t,h-1}]=0 \quad \text{and} \quad \norm{\epsilon_h^t}_{\cH_\psi} \leq 2 \quad \text{a.s.}
\end{equation*}
We overload notation to define, for each pair $(t,h)$, the operator $\operator M_{t,h}=\operator M_t+\sum_{j\leq h}\phi(s^t_j,a^t_j)\otimes \phi(s^\tau_j,a^\tau_j)$, where $\operator M_t=\sum_{\tau < t,h\leq H}\phi(s^\tau_h,a^\tau_h)\otimes \phi(s^\tau_h,a^\tau_h)+\lambda \operator I$.
Note that $\operator M_{t,h} = \operator M_{t,h-1}+\phi(s_h^t,a_h^t)\otimes \phi(s_h^t,a_h^t)$, where $\operator M_{t,0}=\operator M_t=\operator M_{t-1,H}$. Define the estimate
\begin{equation*}
    \hat\Theta_{t,h} = \left(\sum_{\tau < t,h\leq H}\psi(s^\tau_{h+1})\otimes\phi(s^\tau_h,a^\tau_h)+\sum_{j \leq h}\psi(s^t_{j+1})\otimes\phi(s^t_j,a^t_j)\right)\operator M_{t,h}^{-1}~.
\end{equation*}
Note that $\hat \Theta_t=\hat \Theta_{t-1,H}$, where $\hat \Theta_t$ is given by
\eqref{eq:cmo-estimate}.
Consider the random variable
\begin{equation*}
    z_{t,h} = \norm{(\hat\Theta_{t,h}-\Theta_P)\operator M_{t,h}^{1/2}}^2_{\HS}=\tr\left((\hat\Theta_{t,h}-\Theta_P)\operator M_{t,h}(\hat\Theta_{t,h}-\Theta_P)^\top\right)~.
\end{equation*}
Define the operator $\operator X_{t,h}=(\hat\Theta_{t,h}-\Theta_P)\operator M_{t,h}$. Note that
\begin{equation*}
\operator X_{t,h}=\sum_{\tau < t,h\leq H}\epsilon_h^\tau\otimes\phi(s_h^\tau,a_h^\tau)+\sum_{j\leq h}\epsilon_j^t\otimes\phi(s_j^t,a_j^t)-\lambda \Theta_P = \operator X_{t,h-1}+\epsilon_h^t\otimes\phi(s_h^t,a_h^t)~,
\end{equation*}
where $\operator X_{t,0}=\operator X_{t-1,H}$.
We then have
\begin{align*}
    z_{t,h}&=\tr\left(\operator X_{t,h} \operator M_{t,h}^{-1}\operator X_{t,h}^\top\right)\\
    &= \underbrace{\tr\left(\operator X_{t,h-1} \operator M_{t,h}^{-1}\operator X_{t,h-1}^\top\right)}_{\text{A}}+2\underbrace{\tr\left(\epsilon_h^t\otimes\phi(s_h^t,a_h^t) \operator M_{t,h}^{-1}\operator X_{t,h-1}^\top\right)}_{\text{B}} + \underbrace{\tr\left(\epsilon_h^t\otimes\phi(s_h^t,a_h^t) \operator M_{t,h}^{-1}\phi(s_h^t,a_h^t)\otimes\epsilon_h^t\right)}_{\text{C}}~.
\end{align*}
Define the variance estimate
$\sigma^2_{\phi,t,h}(s,a) = \lambda \inner{\phi(s,a)}{\operator M_{t,h}^{-1}\phi(s,a)}_{\cH_\phi}$. Note that $\sigma^2_{\phi,t}(s,a)=\sigma^2_{\phi,t-1,H}(s,a)$. Now, the Sherman-Morrison formula yields
\begin{align*}
    \operator M_{t,h}^{-1}=(\operator M_{t,h-1}+\phi(s_h^t,a_h^t)\otimes\phi(s_h^t,a_h^t))^{-1} &= \operator M_{t,h-1}^{-1}-\frac{\operator M_{t,h-1}^{-1}\phi(s_h^t,a_h^t)\otimes\phi(s_h^t,a_h^t)\operator M_{t,h-1}^{-1}}{1+\inner{\phi(s_h^t,a_h^t)}{\operator M_{t,h-1}^{-1}\phi(s_h^t,a_h^t)}_{\cH_\phi}}\\
    &=\operator M_{t,h-1}^{-1}-\frac{\operator M_{t,h-1}^{-1}\phi(s_h^t,a_h^t)\otimes\phi(s_h^t,a_h^t)\operator M_{t,h-1}^{-1}}{1+\lambda^{-1}\sigma^2_{\phi,t,h-1}(s_h^t,a_h^t)}~.
\end{align*}
We thus have
\begin{align*}
    \text{A}&=\tr\left(\operator X_{t,h-1} \operator M_{t,h-1}^{-1}\operator X_{t,h-1}^\top\right)-\tr\left(\operator X_{t,h-1} \frac{\operator M_{t,h-1}^{-1}\phi(s_h^t,a_h^t)\otimes\phi(s_h^t,a_h^t)\operator M_{t,h-1}^{-1}}{1+\lambda^{-1}\sigma^2_{\phi,t,h-1}(s_h^t,a_h^t)}\operator X_{t,h-1}^\top\right)\nonumber\\ &\leq \tr\left(\operator X_{t,h-1} \operator M_{t,h-1}^{-1}\operator X_{t,h-1}^\top\right) = z_{t,h-1}~,\\
    \begin{split}
        \text{B}&= \tr\Bigg(\epsilon_h^t\otimes\phi(s_h^t,a_h^t) \operator M_{t,h-1}^{-1}\operator X_{t,h-1}^\top\\
        &\quad -\frac{\epsilon_h^t\otimes\phi(s_h^t,a_h^t)\operator M_{t,h-1}^{-1}\phi(s_h^t,a_h^t)\otimes\phi(s_h^t,a_h^t)\operator M_{t,h-1}^{-1}\operator X_{t,h-1}^\top}{1+\lambda^{-1}\sigma^2_{\phi,t,h-1}(s_h^t,a_h^t)}\Bigg)\nonumber\\
        &=\tr\left(\vphantom{\frac{\inner{A}{\hat B^\top}_{\cH_\phi}}{\sigma^2_\phi}} \epsilon_h^t\otimes\phi(s_h^t,a_h^t) (\hat\Theta_{t,h-1}-\Theta_P)^\top\right.\\
        &\quad\left.-\frac{\inner{\phi(s_h^t,a_h^t)}{\operator M_{t,h-1}^{-1}\phi(s_h^t,a_h^t)}_{\cH_\phi}\epsilon_h^t\otimes\phi(s_h^t,a_h^t)(\hat\Theta_{t,h-1}-\Theta_P)^\top}{1+\lambda^{-1}\sigma^2_{\phi,t,h-1}(s_h^t,a_h^t)}\right)\nonumber\\
        &=\left(1-\frac{\lambda^{-1}\sigma^2_{\phi,t,h-1}(s_h^t,a_h^t)}{1+\lambda^{-1}\sigma^2_{\phi,t,h-1}(s_h^t,a_h^t)}\right)\tr\left(\epsilon_h^t\otimes\phi(s_h^t,a_h^t) (\hat\Theta_{t,h-1}-\Theta_P)^\top\right)\nonumber\\
        &=\frac{\inner{(\hat\Theta_{t,h-1}-\Theta_P)\phi(s_h^t,a_h^t)}{\epsilon_h^t}_{\cH_\psi}}{1+\lambda^{-1}\sigma^2_{\phi,t,h-1}(s_h^t,a_h^t)}~,\quad \text{and}
    \end{split}\\
    \begin{split}
        \text{C}&= \norm{\epsilon_h^t}_{\cH_\psi}^2\Bigg(\inner{\phi(s_h^t,a_h^t)}{\operator M_{t,h-1}^{-1}\phi(s_h^t,a_h^t)}_{\cH_\phi}\\
        &\qquad -\inner{\phi(s_h^t,a_h^t)}{\frac{\operator M_{t,h-1}^{-1}\phi(s_h^t,a_h^t)\otimes\phi(s_h^t,a_h^t)\operator M_{t,h-1}^{-1}}{1+\lambda^{-1}\sigma^2_{\phi,t,h-1}(s_h^t,a_h^t)}\phi(s_h^t,a_h^t)}_{\cH_\phi}\Bigg)\nonumber\\
    &=\norm{\epsilon_h^t}_{\cH_\psi}^2\left(\lambda^{-1}\sigma^2_{\phi,t,h-1}(s_h^t,a_h^t)-\frac{\lambda^{-2}\sigma^4_{\phi,t,h-1}(s_h^t,a_h^t)}{1+\lambda^{-1}\sigma^2_{\phi,t,h-1}(s_h^t,a_h^t)}\right)\\
    &=\norm{\epsilon_h^t}_{\cH_\psi}^2\frac{\lambda^{-1}\sigma^2_{\phi,t,h-1}(s_h^t,a_h^t)}{1+\lambda^{-1}\sigma^2_{\phi,t,h-1}(s_h^t,a_h^t)}~.
    \end{split}
\end{align*}
Putting these together, we have
\begin{align*}
    z_{t,h} &\leq z_{t,h-1}+2\frac{\inner{(\hat\Theta_{t,h-1}-\Theta_P)\phi(s_h^t,a_h^t)}{\epsilon_h^t}_{\cH_\psi}}{1+\lambda^{-1}\sigma^2_{\phi,t,h-1}(s_h^t,a_h^t)}+\norm{\epsilon_h^t}_{\cH_\psi}^2\frac{\lambda^{-1}\sigma^2_{\phi,t,h-1}(s_h^t,a_h^t)}{1+\lambda^{-1}\sigma^2_{\phi,t,h-1}(s_h^t,a_h^t)}\\
    &\leq \lambda \norm{\Theta_P}_{\HS}^2+2\sum_{(\tau,j)\leq (t,h)}\frac{\inner{(\hat\Theta_{\tau,j-1}-\Theta_P)\phi(s_j^\tau,a_j^\tau)}{\epsilon_j^\tau}_{\cH_\psi}}{1+\lambda^{-1}\sigma^2_{\phi,\tau,j-1}(s_j^\tau,a_j^\tau)}\\
    &\quad+\sum_{(\tau,j)\leq (t,h)}\norm{\epsilon_j^\tau}_{\cH_\psi}^2\frac{\lambda^{-1}\sigma^2_{\phi,\tau,j-1}(s_j^\tau,a_j^\tau)}{1+\lambda^{-1}\sigma^2_{\phi,\tau,j-1}(s_j^\tau,a_j^\tau)}~.
\end{align*}
For any $B_P \geq \norm{\Theta_P}_{\HS}$, define
\begin{equation*}
    \beta_{t,h}(\delta):= \sqrt{2\lambda B_P^2+ 256(1+\lambda^{-1})\log (\det( \operator I+\lambda^{-1}\operator K_{\phi,t,h})^{1/2})\log(2t^2H/\delta)}\,,
\end{equation*}
where $\operator K_{\phi,t,h}:=[k_{\phi}((s_j^\tau,a_j^\tau),(s_{j'}^{\tau'},a_{j'}^{\tau'}))]_{(\tau,j),(\tau',j')\leq (t,h)}$ denotes the gram-matrix at step $h$ of episode $t$. Note that $\operator K_{\phi,t}=\operator K_{\phi,t-1,H}$. Now define an event $\cE_{t,h}$ as
\begin{equation*}
    \cE_{t,h} = \mathbb{I}\lbrace z_{\tau,j} \leq \beta^2_{\tau,j}(\delta),\quad \forall (\tau,j) \leq (t,h) \rbrace~.
\end{equation*}
Note that $\cE_{t,h}=1$ implies $\cE_{t,h-1}=1$ for all $(t,h) \geq (1,1)$, where $\cE_{t,0}=\cE_{t-1,H}$.
Now we define a sequence of random variables $\lbrace y_{t,h} \rbrace_{t ,h}$ as
$y_{t,h}= \cE_{t,h-1}\frac{\inner{(\hat\Theta_{t,h-1}-\Theta_P)\phi(s_h^t,a_h^t)}{\epsilon_h^t}_{\cH_\psi}}{1+\lambda^{-1}\sigma^2_{\phi,t,h-1}(s_h^t,a_h^t)}$. Note that $y_{t,h}$ is $\cF_{t,h}$-measurable and $\esp[y_{t,h}|\cF_{t,h-1}]=0$.
Hence $\lbrace y_{t,h}\rbrace_{t,h}$ is a martingale difference sequence w.r.t. the filtration $\lbrace \cF_{t,h}\rbrace_{t,h}$. Note that
\begin{align*}
    |y_{t,h}| &\leq \cE_{t,h-1} \norm{\epsilon_h^t}_{\cH_\psi}\frac{\norm{(\hat\Theta_{t,h-1}-\Theta_P) \operator M_{t,h-1}^{1/2}}_{\HS}\norm{\operator M_{t,h-1}^{-1/2}\phi(s_h^t,a_h^t)}_{\cH_\phi}}{1+\lambda^{-1}\sigma^2_{\phi,t,h-1}(s_h^t,a_h^t)}\nonumber\\
    & \leq \beta_{t,h-1}(\delta) \norm{\epsilon_h^t}_{\cH_\psi}\frac{\lambda^{-1}\sigma_{\phi,t,h-1}(s_h^t,a_h^t)}{1+\lambda^{-1}\sigma^2_{\phi,t,h-1}(s_h^t,a_h^t)} \leq 2\beta_{t,h-1}(\delta) \lambda^{-1}\sigma_{\phi,t,h-1}(s_h^t,a_h^t)~,
\end{align*}
since $\norm{\epsilon_h^t}_{\cH_\psi} \leq 2$ a.s.
We then have
\begin{align*}
    \sum_{(\tau,j)\leq (t,h)}|y_{\tau,j}|^2
    &\leq \sum_{(\tau,j)\leq (t,h)}4\beta^2_{\tau,j-1}(\delta)\lambda^{-1}\sigma^2_{\phi,\tau,j-1}(s_j^\tau,a_j^\tau)\\
    &\leq 8(1+\lambda^{-1})\beta^2_{t,h}(\delta)\log (\det( \operator I+\lambda^{-1}\operator K_{\phi,t,h})^{1/2})~,
\end{align*}
where we have used that $\sum_{(\tau,j)\leq (t,h)}\lambda^{-1}\sigma^2_{\phi,\tau,j-1}(s_j^\tau,a_j^\tau)\leq 2(1+\lambda^{-1})\log (\det( \operator I+\lambda^{-1}\operator K_{\phi,t,h})^{1/2})$.
Then, by Azuma-Hoeffding's inequality, with probability at least $1-\frac{\delta}{2t^2H}$, we have
    \begin{equation*}
        \sum_{(\tau,j)\leq (t,h)}y_{\tau,j} \leq \sqrt{16(1+\lambda^{-1})\beta^2_{t,h}(\delta)\log (\det( \operator I+\lambda^{-1}\operator K_{\phi,t,h})^{1/2})\log(2t^2H/\delta)} \leq \beta^2_{t,h}(\delta)/4~,
    \end{equation*}
    as $\beta^2_{t,h}(\delta) \geq 256(1+\lambda^{-1})\log (\det( \operator I+\lambda^{-1}\operator K_{\phi,t,h})^{1/2})\log(2t^2H/\delta)$. Hence, by an union bound
    \begin{equation}\label{eq:conc}
    \prob{\exists\; (t,h) \geq (1,1): \sum\nolimits_{(\tau,j)\leq (t,h)}y_{\tau,j} > \beta^2_{t,h}(\delta)/4}\leq \sum\nolimits_{t=1}^{\infty}\sum\nolimits_{h=1}^{H}\frac{\delta}{2t^2H}\leq \frac{\delta\pi^2}{12}\leq \delta.
\end{equation} 
Now, it suffices to show that $ z_{t,h} \leq \beta^2_{t,h}(\delta)$ for all $(t,h) \geq (1,0)$ given $\sum_{(\tau,j)\leq (t,h)}y_{\tau,j} \leq \beta^2_{t,h}(\delta)/4$, for all $(t,h) \geq (1,1)$. We will show this by induction on $(t,h)$. For the base case $(t,h)=(1,0)$, we have $z_{1,0}=\lambda\norm{\Theta_P}_{\HS}^2 \leq \beta^2_{1,0}(\delta)$. Now by inductive hypothesis, let $z_{\tau,j} \leq \beta^2_{\tau,j}(\delta)$ for all $(1,0) \leq (\tau,j) \leq (t,h-1)$. We then have $\cE_{\tau,j}=1$ for all $(1,0) \leq (\tau,j) \leq (t,h-1)$. Therefore, we have
    \begin{align*}
        z_{t,h} &\leq \lambda \norm{\Theta_P}_{\HS}^2+2\!\!\!\!\sum_{(\tau,j)\leq (t,h)}\!\!\!\!\cE_{\tau,j-1}\frac{\inner{(\hat\Theta_{\tau,j-1}-\Theta_P)\phi(s_j^\tau,a_j^\tau)}{\epsilon_j^\tau}_{\cH_\psi}}{1+\lambda^{-1}\sigma^2_{\phi,\tau,j-1}(s_j^\tau,a_j^\tau)}\\
        &\quad +\!\!\!\!\sum_{(\tau,j)\leq (t,h)}\!\!\!\!\norm{\epsilon_j^\tau}_{\cH_\psi}^2\frac{\lambda^{-1}\sigma^2_{\phi,\tau,j-1}(s_j^\tau,a_j^\tau)}{1+\lambda^{-1}\sigma^2_{\phi,\tau,j-1}(s_j^\tau,a_j^\tau)}\\
        &\leq \lambda \norm{\Theta_P}_{\HS}^2+2\sum_{(\tau,j)\leq (t,h)}y_{\tau,j}+4\sum_{(\tau,j)\leq (t,h)}\frac{\lambda^{-1}\sigma^2_{\phi,\tau,j-1}(s_j^\tau,a_j^\tau)}{1+\lambda^{-1}\sigma^2_{\phi,\tau,j-1}(s_j^\tau,a_j^\tau)}\\
        &\leq  \lambda \norm{\Theta_P}_{\HS}^2 + \beta^2_{t,h}(\delta)/2+8(1+\lambda^{-1})\log (\det( \operator I+\lambda^{-1}\operator K_{\phi,t,h})^{1/2}) \leq \beta^2_{t,h}(\delta)~,
    \end{align*}
    as $\beta^2_{t,h}(\delta) \geq 2\lambda\norm{\Theta_P}_{\HS}^2+16(1+\lambda^{-1})\log (\det( \operator I+\lambda^{-1}\operator K_{\phi,t,h})^{1/2})$. Now the proof follows from~\autoref{eq:conc} and noting that
    $\norm{(\hat\Theta_{t}-\Theta_P)\operator M_{t}^{1/2}}_{\HS}=z_{t-1,H}^{1/2}\leq \beta_{t-1,H}(\delta) \leq \beta_t(\delta)$.

\subsection{Regret analysis of \CMERL}
To prove the regret bound in \autoref{thm:cumulative-regret-final}, we establish a sequence of intermediate results to bound the performance gap between the optimal policy and the policy followed by \CMERL. In \autoref{lem:optimism}, we start with a control on the $Q$-function estimates $Q_h^t$, which in turn leads to the regret bound, as our policy is based on $Q_h^t$. The result implies that as long as the true transition distribution lies in the confidence set $C_t$, the $Q$-updates are optimistic estimates of the optimal $Q$-values and thus, allow us to pick an optimistic action while sufficiently exploring the state space. 
\begin{mylemma}[Optimism]\label{lem:optimism}
Let $P \in \cC_t$. Then, $Q^\star_h(s, a) \leq Q^t_h(s,a)$ for all $h$, and $(s,a)$\,.
\end{mylemma}
\begin{proof}
We prove the lemma by induction on $h$. When $h = H$, the inequality holds by definition. Now we assume that the lemma holds for some $h'=h+1$, where $1 \leq h < H$. This implies that for all $s \in \cS$,
\beqn
V_{h+1}^t(s) = \min \left\lbrace H, \max_{a \in \cA}Q_{h+1}^t(s,a)]\right \rbrace \geq \min \left\lbrace H, \max_{a \in \cA}Q_{h+1}^\star(s,a)]\right \rbrace = V^{\star}_{h+1}(s)~.
\eeqn
We then have, for all $(s,a) \in \cS \times \cA$, that
\beqan
Q^\star_{h}(s,a) &=& R(s,a) + \esp_{X \sim P(\cdot|s,a)}\left[V_{h+1}^{\star}(X)\right]\\
& \leq &  R(s,a) + \esp_{X \sim P(\cdot|s,a)}\left[V_{h+1}^{t}(X)\right]\\
&\leq &  R(s,a) + \max_{\Theta_{P'}\in \cC_t}\esp_{X \sim P'(\cdot|s,a)}\left[V_{h+1}^{\star}(X)\right] \leq Q_{h}^t(s,a)~,
\eeqan
where the third step follows from $\Theta_P \in \cC_t$.
\end{proof}
\begin{mylemma}[Gap between optimistic and actual values]\label{lem:gap-optimisitc-actual-value}
Let $g_h^t(s)=V_h^t(s)-V_h^{\pi^t}(s)$, and $m_h^t=\esp_{X \sim P(\cdot|s_h^t,a_h^t)}\left[g_{h+1}^t(X)\right]-g_{h+1}^{t}(s_{h+1}^t)$. Then
\begin{align*}
 g_1^t(s_1^t) \leq \sum\nolimits_{h=1}^{H}Q_h^t(s_h^t,a_h^t)-\left(R(s_h^t,a_h^t)+\esp_{X \sim P(\cdot|s_h^t,a_h^t)}\left[V_{h+1}^{t}(X)\right]\right) + \sum\nolimits_{h=1}^{H}m_h^t~.  
\end{align*}
\end{mylemma}
\begin{proof}
Note that $a_h^t=\pi_h^t(s_h^t)=\argmax_{a \in \cA}Q_h^t(s_h^t,a)$. Therefore
\beqan
V_h^{\pi^t}(s_h^t) &=&  R(s_h^t,a_h^t)+\esp_{X \sim P(\cdot|s_h^t,a_h^t)}\left[V_{h+1}^{\pi^t}(X)\right]\\
& = & R(s_h^t,a_h^t)+\esp_{X \sim P(\cdot|s_h^t,a_h^t)}\left[V_{h+1}^t(X)\right] - \esp_{X \sim P(\cdot|s_h^t,a_h^t)}\left[g_{h+1}^t(X)\right]\\
& = & R(s_h^t,a_h^t)+\esp_{X \sim P(\cdot|s_h^t,a_h^t)}\left[V_{h+1}^t(X)\right] - g_{h+1}^{t}(s_{h+1}^t)-m_h^t~.
\eeqan
We also have $
V_h^t(s_h^t) = \min \left \lbrace H, \max_{a \in \cA}Q_h^t(s_h^t,a)\right\rbrace= \min \left \lbrace H, Q_h^t(s_h^t,a_h^t)\right\rbrace \leq Q_h^t(s_h^t,a_h^t)$.
Therefore, $
g_h^t(s_h^t) \leq Q_h^t(s_h^t,a_h^t)-\left(R(s_h^t,a_h^t)+\esp_{X \sim P(\cdot|s_h^t,a_h^t)}\left[V_{h+1}^{t}(X)\right]\right) + g_{h+1}^{t}(s_{h+1}^t)+m_h^t$.
Since $g_{H+1}^t(s)=0$ for all $s \in \cS$, a simple recursion over all $h \in [H]$ completes the proof.
\end{proof}

\begin{mylemma}[Cumulative regret expressed through $Q$-estimates] \label{lem:cumulative-regret}
Let $\Theta_P \in \cC_t$ for all $t \geq 1$. Then for any $\delta \in (0,1]$, the following holds with probability at least $1-\delta/2 :$
\beqn
\cR(N) \leq  \sum\nolimits_{t \leq T,h\leq H}Q_h^t(s_h^t,a_h^t)-\left(R(s_h^t,a_h^t)+\esp_{X \sim P(\cdot|s_h^t,a_h^t)}\left[V_{h+1}^{t}(X)\right]\right) + 2H\sqrt{2N\log(2/\delta)}~.
\eeqn
\end{mylemma}
\begin{proof}
If $\Theta_P \in \cC_t$ for all $t$, we have from Lemma \ref{lem:optimism} that
\beqn
\forall t \geq 1, \quad V_1^t(s_1^t) = \min \left\lbrace H, \max_{a \in \cA}Q_1^t(s_1^t,a)\right \rbrace \geq \min \left\lbrace H, \max_{a \in \cA}Q_1^\star(s_1^t,a)\right \rbrace = V^{\star}_1(s_1^t)~.
\eeqn
Therefore the cumulative regret after $N=TH$ steps is given by
\beqn
\cR(N) =\sum\nolimits_{t=1}^{T} \left[V^{\star}_{1}(s_1^t)-V^{\pi^t}_{1}(s_1^t)\right] \leq \sum\nolimits_{t=1}^{T}\left(V_1^t(s_1^t)-V_1^{\pi^t}(s_1^t)\right)=\sum\nolimits_{t=1}^{T}g_1^t(s_1^t)~.
\eeqn
We then have from Lemma \ref{lem:gap-optimisitc-actual-value} that
\beqn
\cR(N) \leq \sum\nolimits_{t \leq T,h\leq H}Q_h^t(s_h^t,a_h^t)-\left(R(s_h^t,a_h^t)+\esp_{X \sim P(\cdot|s_h^t,a_h^t)}\left[V_{h+1}^{t}(X)\right]\right) + \sum\nolimits_{t \leq T,h \leq H}m_h^t~.
\eeqn
Note that $(m_h^t)_{t,h}$ is a martingale difference sequence adapted to the filtration $\cF_{t,h}$ with $|m_h^t|\leq 2H$. Hence, by Azuma-Hoeffding inequality, with probability at least $1-\delta/2$, $
\sum_{t\leq T,h \leq H}m_h^t \leq 2H\sqrt{2TH\log(2/\delta)}=2H\sqrt{2N\log(2/\delta)}$,
which proves the result.
\end{proof}

\begin{mylemma}[Error in $Q$-estimates]\label{lem:error-in-value}
Let $\Theta_P \in \cC_t$.
Then, for all $h \leq H$ and $(s,a) \in \cS \times \cA$, 
\beqan
Q_h^t(s,a)-\left(R(s,a)+\esp_{X \sim P(\cdot|s,a)}\left[V_{h+1}^{t}(X)\right]\right)
 \leq 2B_V\beta_t(\delta/2)\lambda^{-1/2} \sigma_{\phi,t}(s,a) ~.
\eeqan
\end{mylemma}
\begin{proof}
For $\Theta_P\in\cC_t$, uniformly over all $t \in \Nat$, the mean embeddings satisfy
\beqan
\norm{\cme_{P}^{(s,a)}-\hat\cme_t^{(s,a)}}_{\cH_\psi} & = &
\norm{(\Theta_{P}-\hat\Theta_t)\phi(s,a)}_{\cH_\psi}\\ &\leq& \norm{(\Theta_{P}-\hat\Theta_t)(\hat{\operator C}_{\phi,t}+\lambda  \operator I)^{1/2}}\norm{(\hat{\operator C}_{\phi,t}+\lambda  \operator I)^{-1/2}\phi(s,a)}_{\cH_\phi}\\
&\leq & \beta_t(\delta/2)\inner{\phi(s,a)}{(\hat{\operator C}_{\phi,t}+\lambda  \operator I)^{-1}\phi(s,a)}_{\cH_\phi}^{1/2}=\beta_t(\delta/2)\lambda^{-1/2} \sigma_{\phi,t}(s,a)~.
\eeqan
Now, by~\autoref{ass:opt-closure}, $V_{h+1}^t \in \cH_\psi$. Hence, the $Q$-estimates (\autoref{eq:efficient-optimistic-planning}) can written as $
Q_h^t(s,a) = R(s,a) + \inner{V_{h+1}^t}{\hat\cme_t^{(s,a)}}_{\cH_\psi} + B_V\beta_t(\delta/2)\lambda^{-1/2} \sigma_{\phi,t}(s,a)$.
Therefore, we have 
\beqan
&&Q_h^t(s,a)-\left(R(s,a)+\esp_{X \sim P(\cdot|s,a)}\left[V_{h+1}^{t}(X)\right]\right)\\
&=& \inner{V_{h+1}^t}{\hat\cme_t^{(s,a)}-\cme_P^{(s,a)}}_{\cH_\psi} + B_V\beta_t(\delta/2)\lambda^{-1/2} \sigma_{\phi,t}(s,a)
\leq   2B_V\beta_t(\delta/2)\lambda^{-1/2} \sigma_{\phi,t}(s,a)~,
\eeqan
where the last step holds since $\norm{V_{h+1}^t}_{\cH_\psi} \leq B_V$ and $\Theta_P \in \cC_t$.
\end{proof}

\begin{mylemma}[Sum of predictive variances]\label{lem:sum-of-predictive-variances}
Let $\sup_{s,a}\sqrt{k\left((s,a),(s,a)\right)} \leq B_{\phi}$. Then
\beqn
\sum\nolimits_{t\leq T,h\leq H} \lambda^{-1} \sigma^2_{\phi,t}(s_h^t,a_h^t) \leq (1+\lambda^{-1}B_\phi^2 H)\log \det ( \operator I+\lambda^{-1}\operator K_{\phi,T+1})~.
\eeqn
\end{mylemma}
\begin{proof}
We have from \autoref{eq:predictive-variance} that $\lambda^{-1} \sigma^2_{\phi,t}(s,a) =\tr \left(\operator M_t^{-1}\phi(s,a)\otimes \phi(s,a)\right)$.
We also note that $\operator M_t \geq \lambda  \operator I$ and $\phi(s,a)\otimes \phi(s,a) \leq B_\phi^2  \operator I$. Therefore $\operator M_t^{-1}\phi(s,a)\otimes \phi(s,a) \leq \lambda^{-1} B_\phi^2  \operator I$. Now, since $\operator M_{t+1}=\operator M_t+\sum_{h \leq H} \phi(s_h^t,a_h^t)\otimes \phi(s_h^t,a_h^t)$, we have
\beqn
\operator M_t^{-1} = \big( \operator I +\operator M_t^{-1}\sum\nolimits_{h \leq H} \phi(s_h^t,a_h^t)\otimes \phi(s_h^t,a_h^t)\big)\operator M_{t+1}^{-1} \preceq (1+\lambda^{-1}B_\phi^2 H)\operator M_{t+1}^{-1}~.
\eeqn
We then have
\beqan
\sum\nolimits_{t\leq T,h\leq H} \lambda^{-1} \sigma^2_{\phi,t}(s_h^t,a_h^t) &\leq & (1+\lambda^{-1}B_\phi^2 H)\sum\nolimits_{t \leq T}\tr \left(\operator M_{t+1}^{-1}(\operator M_{t+1}-\operator M_t)\right)\\
& \leq & (1+\lambda^{-1}B_\phi^2 H)\sum\nolimits_{t \leq T}  \log \frac{\det(\operator M_{t+1})}{\det(\operator M_t)}\\
& = & (1+\lambda^{-1}B_\phi^2 H)\log \det ( \operator I+\lambda^{-1}\operator M_{T+1})\\
& = & (1+\lambda^{-1}B_\phi^2 H)\log \det ( \operator I+\lambda^{-1}\operator K_{\phi,T+1})~.
\eeqan
Here, in the second step we have used that for two positive definite operators $\operator A$ and $\operator B$ such that $\operator A-\operator B$ is positive
semi-definite, $\tr(\operator A^{-1}(\operator A-\operator B)) \leq  \log \frac{\det(\operator A)}{\det(\operator B)}$. The last step follows from Sylvester's determinant identity.
\end{proof}
Based on the results in this section, we can finally derive a proof for \autoref{thm:cumulative-regret-final}.
\subsubsection{Proof of \autoref{thm:cumulative-regret-final}}
We have from Lemma \ref{lem:cumulative-regret} and \ref{lem:error-in-value} that if $\Theta_P \in \cC_t$ for all $t$, then with probability at least $1-\delta/2$, the cumulative regret
\beqan
\cR(N)
&\leq & 2B_V\beta_T(\delta/2) \sum\nolimits_{t\leq T,h\leq H}\lambda^{-1/2} \sigma_{\phi,t}(s_h^t,a_h^t) + 2H\sqrt{2N\log(2/\delta)}\\
& \leq & 2B_V\alpha_{N,\delta} \sqrt{TH\sum\nolimits_{t\leq T,h\leq H}\lambda^{-1} \sigma^2_{\phi,t}(s_h^t,a_h^t)} + 2H\sqrt{2N\log(2/\delta)}\\
&\leq & 2B_V\alpha_{N,\delta} \sqrt{2(1+\lambda^{-1}B_\phi^2 H)N\gamma_N} + 2H\sqrt{2N\log(2/\delta)}~.
\eeqan
The first step follows since $\beta_t(\delta)$ increases with $t$, the second step is due to Cauchy-Schwartz's inequality and the fact that $\beta_T(\delta/2) \leq \alpha_{N,\delta}$, and the final step follows from \autoref{lem:sum-of-predictive-variances}. The proof now can be completed using \autoref{lem:concentration} and taking a union bound.

\subsection{Regret analysis of \CMERL\ under model misspecification}

To prove the regret bound in \autoref{thm:cumulative-regret-final-approx}, we follow the similar arguments used in proving \autoref{thm:cumulative-regret-final}, but with necessary modifications taking the effect of the misspecification error $\zeta$ into account. We first derive the following result.


\begin{mylemma}[Error in approximate Q-values]
\label{lem:approx-Q-values}
Let $\Theta_P \in \cC_t$. Then, for all $h \in [H]$ and $(s,a) \in \cS \times \cA$, we have that:
\beqan
Q^t_h(s,a) - \left(R(s,a)+\esp_{X \sim P(\cdot | s,a)}\left[V_{h+1}^{t}(X)\right]\right) 
\leq  2\left(B_V + \zeta\norm{  1}_{\cH_\psi}  \right) \lambda^{-1/2}\beta_t(\delta/2)\sigma_{\phi,t}(s,a) + 2\zeta~.
\eeqan
\end{mylemma}

\begin{proof}
Note that the $Q$-estimates (\autoref{eq:efficient-optimistic-planning-approx}) can rewritten as
\begin{equation*}
    Q_h^t(s,a) := R(s,a) +  \alpha_t(s,a)^\top  v_{h+1}^t + \left(B_V + \zeta\norm{  1}_{\cH_\psi}\right)\lambda^{-1/2}\beta_t(\delta/2)\sigma_{\phi,t}(s,a)~.
\end{equation*}
By \autoref{ass:approx-opt-closure}, there exists a function $\tilde V_{h+1}^t \in \cH_\psi$ such that $\|V_{h+1}^t-\tilde V_{h+1}^t\|_{\infty} \leq \zeta$. We now define the vector $ {\tilde{v}}_{h+1}^t := [\tilde V_{h+1}^t(s_{h'+1}^{\tau})]_{\tau<t, h'\leq H}$ and introduce the shorthand notation $\esp_{P(\cdot | s,a)}[f] := \esp_{X\sim P(\cdot | s,a)}[f(X)]$.
We then have
\begin{align}
        &\left| \alpha_t(s,a)^\top  {v}_{h+1}^t - \esp_{X \sim P(\cdot | s,a)}[V_{h+1}^t(X)]\right|\nonumber\\
        &= \left| \alpha_t(s,a)^\top  {\tilde{v}}_{h+1}^t +  \alpha_t(s,a)^\top( {v}_{h+1}^t -  {\tilde{v}}_{h+1}^t) - \esp_{P(\cdot | s,a)}[\tilde{V}_{h+1}^t] + \esp_{P(\cdot | s,a)}[\tilde{V}_{h+1}^t - V_{h+1}^t]\right|\nonumber\\
        &\leq \left| \alpha_t(s,a)^\top  {\tilde{v}}_{h+1}^t - \esp_{P(\cdot | s,a)}[\tilde{V}_{h+1}^t]\right| + \norm{ \alpha_t(s,a)}_1\norm{V_{h+1}^t - \tilde{V}_{h+1}^t}_{\infty}  + \norm{\tilde{V}_{h+1}^t - V_{h+1}^t}_{\infty}\nonumber\\
        &\leq \left| \alpha_t(s,a)^\top  {\tilde{v}}_{h+1}^t  - \esp_{P(\cdot | s,a)}[\tilde{V}_{h+1}^t]\right| +  \zeta\left(1+\norm{ \alpha_t(s,a)}_1\right)~,
    \label{eq:approx-Q-diff}
\end{align}
which follows by an application of H\"{o}lder's inequality. Now, as $\Theta_P \in \cC_t$ and $\norm{\tilde V_{h+1}^t}_{\Hilbert_\psi}\leq B_V$, the following also holds:
\begin{align}
        \left| \alpha_t(s,a)^\top  {\tilde{v}}_{h+1}^t - \esp_{P(\cdot | s,a)}[\tilde{V}_{h+1}^t]\right| &= \left|\inner{\tilde{V}_{h+1}^t}{\hat\cme_t(s,a) - \cme_P(s,a)}_{\cH_\psi}\right|\nonumber\\
        &\leq B_V \lambda^{-1/2}\beta_t(\delta/2)\sigma_{\phi,t}(s,a)~.
    \label{eq:alpha-value-bound}
\end{align}
By \autoref{ass:constant-function-in-RKHS}, the constant function $  1:\StateSpace\to\Real$ is an element of $\Hilbert_\psi$. For stationary (radial) kernels $k_\phi$, we have $[ \alpha_t(s,a)]_{\tau, h} \geq 0$, $\forall \tau\leq t, h\leq H$. Now, as $\Theta_P \in \cC_t$, we have:
\begin{align}
    \norm{ {\alpha}_t(s,a)}_1 = \inner{\hat\cme_t^{
    (s,a)}}{  1}_{\cH_\psi}
    & = \inner{\cme_P^{
    (s,a)}}{  1}_{\cH_\psi}+ \inner{\hat\cme_t^{
    (s,a)}-\cme_P^{(s,a)}}{  1}_{\cH_\psi}\nonumber\nonumber\\
    &\leq \Esp_{X\sim P(\cdot|s,a)}[  1(X)] + \norm{  1}_{\cH_\psi}\norm{\hat\cme_t^{(s,a)}-\cme_P^{(s,a)}}_{\cH_\psi}\nonumber\\
    &= 1 + \norm{  1}_{\cH_\psi}\lambda^{-1/2}\beta_t(\delta/2)\sigma_{\phi,t}(s,a)~.
\label{eq:alpha-norm-bound}
\end{align}
Combining \autoref{eq:alpha-value-bound} and \autoref{eq:alpha-norm-bound} with \autoref{eq:approx-Q-diff} yields:
\begin{align}
  & \left| \alpha_t(s,a)^\top  {v}_{h+1}^t - \esp_{X \sim P(\cdot | s,a)(X)}[V_{h+1}^t]\right|\nonumber\\ &\leq B_V\lambda^{-1/2}\beta_t(\delta/2)\sigma_{\phi,t}(s,a) + \zeta \left(2 + \norm{  1}_{\cH_\psi}\lambda^{-1/2}\beta_t(\delta/2)\sigma_{\phi,t}(s,a)\right)\nonumber\\
        &\leq \left( B_V + \zeta\norm{  1}_{\cH_\psi}  \right) \lambda^{-1/2}\beta_t(\delta/2)\sigma_{\phi,t}(s,a) + 2\zeta~.
\label{eq:expectation-approximation}  
\end{align}
Finally, the result follows by noting that
\begin{equation*}
\begin{split}
    & Q^t_h(s,a) - \left(R(s,a)+\esp_{X \sim P(\cdot | s,a)}\left[V_{h+1}^{t}(X)\right]\right) \\
    &\leq   \alpha_t(s,a)^\top  {v}_{h+1}^t - \esp_{P(\cdot | s,a)}[V_{h+1}^t] + \left(B_V + \zeta\norm{  1}_{\cH_\psi}  \right) \lambda^{-1/2}\beta_t(\delta/2)\sigma_{\phi,t}(s,a)\\
    &\leq  2\left(B_V + \zeta\norm{  1}_{\cH_\psi}  \right) \lambda^{-1/2}\beta_t(\delta/2)\sigma_{\phi,t}(s,a) + 2\zeta~,
\end{split}
\end{equation*}
which concludes the proof.
\end{proof}
With the result above, we can derive an upper bound on the optimal value $Q^\star_h$ as follows.
\begin{mylemma}
\label{lem:approximate-optimism}
Let $\Theta_P \in \cC_t$. Then $
\forall h \in [H],\; \forall (s,a) \in \cS \times \cA,\quad Q^\star_h(s, a) \leq Q^t_h(s,a) + 2(H-h)\zeta$.
\end{mylemma}
\begin{proof}
We prove the lemma by induction on $h$. When $h = H$, the result holds by definition. Now assume that it holds for some $h'=h+1$, where $1\leq h < H$. This implies that
\begin{align*}
\forall s\in\cS,\quad V_{h+1}^t(s) &= \min \left\lbrace H, \max_{a \in \cA}Q_{h+1}^t(s,a)]\right \rbrace\\
&\geq  \min \left\lbrace H, \max_{a \in \cA}Q_{h+1}^\star(s,a)]\right \rbrace -2(H-h-1)\zeta= V^{\star}_{h+1}(s)-2(H-h-1)\zeta~.    
\end{align*}
We then have, for all $(s,a) \in \cS \times \cA$, that
\beqa
Q^\star_{h}(s,a) &=& R(s,a) + \esp_{X \sim P(\cdot|s,a)}\left[V_{h+1}^{\star}(X)\right]\nonumber\\
& \leq &  R(s,a) + \esp_{X \sim P(\cdot|s,a)}\left[V_{h+1}^{t}(X)\right]+2(H-h-1)\zeta~.
\label{eq:induction-unnormalised}
\eeqa
Using \autoref{eq:expectation-approximation} in the proof of \autoref{lem:approx-Q-values}, we now see that
\begin{equation*}
    \esp_{X \sim P(\cdot | s,a)}[V_{h+1}^t(X)] \leq  \alpha_t(s,a)^\top  {v}_{h+1}^t + 
    \left( B_V + \zeta\norm{  1}_{\cH_\psi}  \right) \lambda^{-1/2}\beta_t(\delta/2)\sigma_{\phi,t}(s,a) + 2\zeta
\end{equation*}
which holds as $\Theta_P \in \cC_t$. We then have from \autoref{eq:induction-unnormalised} that
\begin{equation*}
    \begin{split}
        Q^\star_{h}(s,a) &\leq R(s,a) +  \alpha_t(s,a)^\top  {v}_{h+1}^t + 
        \left( B_V + \zeta\norm{  1}_{\cH_\psi}  \right) \lambda^{-1/2}\beta_t(\delta/2)\sigma_{\phi,t}(s,a) + 2(H-h)\zeta\\
        &= Q^t_{h}(s,a) + 2(H-h)\zeta~,
    \end{split}
\end{equation*}
which follows from the definition of $Q$-estimates.
\end{proof}
Given \autoref{lem:approx-Q-values} and \autoref{lem:approximate-optimism}, we can finally  prove \autoref{thm:cumulative-regret-final-approx}.
\subsubsection{Proof of \autoref{thm:cumulative-regret-final-approx}}

If $\Theta_P \in \cC_t$ for all $t \geq 1$, we have from \autoref{lem:approximate-optimism} that $V_1^\star(s_1^t) \leq V_1^t(s_1^t) + 2(H-1)\zeta$.
Then following similar steps as in the proof of \autoref{thm:cumulative-regret-final} and \autoref{lem:cumulative-regret}, with probability at least $1-\delta/2$, we have the following:
\begin{align*}
 \cR(N) &\leq \sum_{t \leq T,h\leq H}Q_h^t(s_h^t,a_h^t)-\left(R(s_h^t,a_h^t)+\esp_{X \sim P(\cdot|s_h^t,a_h^t)}\left[V_{h+1}^{t}(X)\right]\right)\\
 &\quad + 2H\sqrt{2N\log(2/\delta)}+2\zeta T(H-1)\\
&\leq 2\left(B_V+\zeta\norm{  1}_{\cH_\psi}\right)\sum_{t \leq T,h\leq H}\lambda^{-1/2}\beta_t(\delta/2)\sigma_{\phi,t}(s_h^t,a_h^t)+2\zeta TH+2H\sqrt{2N\log(2/\delta)}\\
&\quad +2\zeta T(H-1)\\
&\leq  2\left(B_V+\zeta\norm{  1}_{\cH_\psi}\right)\beta_T(\delta/2) \sum_{t\leq T,h\leq H}\lambda^{-1/2} \sigma_{\phi,t}(s_h^t,a_h^t) + 4\zeta TH  + 2H\sqrt{2N\log(2/\delta)}\\
& \leq  2\left(B_V+\zeta\norm{  1}_{\cH_\psi}\right)\alpha_{N,\delta} \sqrt{TH\sum_{t\leq T,h\leq H}\lambda^{-1} \sigma^2_{\phi,t}(s_h^t,a_h^t)} + 4\zeta TH  + 2H\sqrt{2N\log(2/\delta)}\\
&\leq  2\left(B_V+\zeta\norm{  1}_{\cH_\psi}\right)\alpha_{N,\delta} \sqrt{2(1+\lambda^{-1}B_\phi^2 H)N\gamma_N} + 4\zeta N  + 2H\sqrt{2N\log(2/\delta)}~.   
\end{align*}
The second step follows from \autoref{lem:approx-Q-values}, the third step from monotonicity of $\beta_t(\delta)$ with $t$, the fourth step is due to Cauchy-Schwartz's inequality and the fact that $\beta_T(\delta/2) \leq \alpha_{N,\delta}$, and the final step follows from \autoref{lem:sum-of-predictive-variances}. The proof now can be completed using \autoref{lem:concentration} and taking a union bound.

\end{document}